\documentclass{article} %
\usepackage{hyperref}
\usepackage[accepted]{icml2020}
\usepackage{times}
\usepackage{multirow}
\usepackage{url}

\usepackage{amsmath}
\usepackage{amsfonts}
\usepackage{amsthm}

\usepackage{algorithm}

\usepackage{textcomp}
\usepackage{mathrsfs}
\usepackage{booktabs}

\usepackage{microtype}
\usepackage{graphicx}
\usepackage{subfigure}
\usepackage{booktabs} %

\usepackage{lipsum} 
\usepackage{graphicx}

\renewcommand{\hat}[1]{\widehat{#1}}

\usepackage{kitstyle}

\newtheorem{theorem}{Theorem}[section]

\newtheorem{lemma}{Lemma}[section]
\newtheorem{proposition}{Proposition}[section]

\usepackage{thmtools}
\usepackage{thm-restate}

\usepackage{hyperref}

\usepackage{cleveref}

\icmltitlerunning{
Accountable Off-Policy Evaluation With Kernel Bellman Statistics
}

\begin{document}
\twocolumn[

\icmltitle{
Accountable Off-Policy Evaluation With Kernel Bellman Statistics
}
\icmlsetsymbol{equal}{*}
\begin{icmlauthorlist}
\icmlauthor{Yihao Feng}{ut}
\icmlauthor{Tongzheng Ren}{equal,ut}
\icmlauthor{Ziyang Tang}{equal,ut}
\icmlauthor{Qiang Liu}{ut}
\end{icmlauthorlist}
\vskip 0.3in
\icmlcorrespondingauthor{Yihao Feng}{yihao@cs.utexas.edu}
\icmlcorrespondingauthor{Tongzheng Ren}{tzren@cs.utexas.edu}

\icmlaffiliation{ut}{Department of Computer Science, The University of Texas at Austin}]

\printAffiliationsAndNotice{\icmlEqualContribution}

\begin{abstract}
We consider off-policy evaluation (OPE), which 
evaluates the performance of a new policy from observed data collected from previous experiments, 
without requiring the execution of the new policy. 
This finds important applications in areas with high execution cost or safety concerns, such as medical diagnosis, recommendation systems and robotics. 
In practice, due to the limited information from off-policy data,  
it is highly desirable to construct rigorous confidence intervals, not just point estimation, for the policy performance. 
In this work, we propose a new \emph{variational} framework 
which reduces the problem of calculating tight confidence bounds in OPE into an optimization problem on a  feasible set that catches the true state-action value function with high probability. 
The feasible set is constructed by leveraging statistical properties of a recently proposed kernel Bellman loss \citep{feng2019kernel}. 
We design an efficient computational approach for calculating our bounds, 
and extend it to perform post-hoc diagnosis and correction for existing estimators. 
Empirical results show that 
our method yields tight confidence intervals in different settings.

\end{abstract}

\section{Introduction}
Reinforcement learning \citep{sutton98beinforcement} has achieved remarkable successes recently, but is still highly data demanding. 
An essential approach to improve data efficiency is to  use \emph{off-policy} methods, which leverage 
 historical data collected from different behavior  policies (a.k.a. off-policy data) to evaluate and optimize the target policies of interest.  
Using off-policy data is of critical importance in many real-world applications \citep[e.g.,][]{murphy01marginal,li11unbiased,bottou13counterfactual,thomas17predictive}, 
especially 
in cases when it is 
infeasible or even risky to collect online experimental data or build high-quality simulators, such as 
robotics, 
advertisement, 
online education, 
and 
medical treatment  \citep[e.g.,][]{murphy01marginal,li11unbiased,hirano2003efficient,liao2019off}.

This work concerns the problem of off-policy  evaluation (OPE), which estimates the expected reward of a given target policy from off-policy data. 
Because the information in off-policy data is often  limited, 
typical point estimation may suffer from large error in difficult cases.  
Therefore, for making high-stakes decisions in areas such as medical treatment,  
it is crucially important %
to provide reliable \emph{confidence bounds} for quantifying the uncertainties in the estimation.  %

Importance sampling (IS) provides a basic principle for developing OPE methods 
\citep[e.g.,][]{dudik11doubly, jiang16doubly,thomas16data,liu2018breaking,nachum2019dualdice,zhang2020gendice}, 
as well as doubly robust estimators  
\citep{dudik11doubly,jiang16doubly, farajtabar18more,kallus2019double,tang2020doubly},
as well as constructing confidence bounds for OPE  \citep[e.g.,][]{thomas15high,hanna2017bootstrapping}.
However, the typical IS-based methods suffer from a  number of critical disadvantages: 

$\bullet$~\emph{Requiring known behavior policies}:
The typical IS estimators require the observed data to be drawn from a known behavior policy. However, this is often not the case in practice. 
In practice, the policies from which the observed data is collected may be unknown, or even not exist (e.g., when the data is drawn from a mix of different policies). 

$\bullet$~\emph{High variance and loose bounds}: 
The variance of IS-based estimators can be excessively high when the target and behavior policies are very different from each other. As a consequence, concentration confidence bounds based on IS \citep[e.g.,][]{thomas15high,thomas2015highope} could be very loose. 

$\bullet$~\emph{Curse of long horizon}: 
The variance of IS estimators are known to deteriorate quickly when the horizon the Markov decision processes becomes large. 
Although there are recent improved methods for long horizon problems \citep{liu2018breaking,tang2020doubly,mousavi2020blackbox}, 
they introduce additional biases due to the estimation of the density ratios.

\textbf{Our Method}\quad
We propose 
a variational framework for constructing  
\emph{tight, provable confidence bounds} 
which avoids both the requirement on known behavior policies and the curse of long horizon.  
Instead of using importance sampling,
we construct bounds based on the state-action value function (a.k.a. Q-function) specified by Bellman equation. 
The idea is to construct a high confidence feasible  set of the true Q-function of the target policy, by leveraging the tail bound of the kernel Bellman loss  \citep{feng2019kernel}, and derive the upper (resp. lower) bounds of the expected reward of interest by solving a maximization (resp. minimization) problem on the high confidence feasible set.  
Further, by constraining the optimization inside a reproducing kernel Hilbert space (RKHS) ball, 
we obtain a computationally efficient framework for providing confidence bounds of  off-policy value evaluation.

\textbf{Post-hoc Analysis}\quad
In practice, 
we may already have an estimated Q-function 
from previous data/experiments, and we want to 
construct post-hoc confidence bounds around the existing estimation, 
or 
perform post-hoc correction if the estimation is detected to deviate significantly from the true value. 
In this work, we enable such post-hoc diagnosis by extending our main framework. 
We empirically evaluate our method on several control benchmarks; results show that our  method can provide  tight confidence bounds, and also reasonable post-hoc corrections for existing estimators.

\section{Related Work}
Off-policy value (point) estimation has been studied in various reinforcement learning settings, including contextual bandits \citep[e.g.,][]{dudik11doubly,wang17optimal}, finite horizon Markov decision processes (MDPs) \citep[e.g.,][]{thomas16data,jiang16doubly,farajtabar18more,xie2019optimal}, 
and infinite horizon RL MDPs  \citep[e.g.,][]{liu2018breaking,nachum2019dualdice,tang2020doubly}.
Unlike our work, most of these works are not behavior agnostic 
and require to know the behavior policy, except a line of very recent works \citep[e.g.,][]{nachum2019dualdice,zhang2020gendice,mousavi2020blackbox}.

A smaller set of works focus on confidence bounds for various RL  problems \citep[e.g.,][]{white2010interval,thomas15high,thomas2015highope,hanna2017bootstrapping,dann2018policy,jiang2020minimax,duan2020minimax}.
A recent line of works \citep[e.g.,][]{dann2018policy,jin2018q,jin2019provably} have utilized upper confidence bounds for explorations on tabular (or linear) MDPs. 
\citet{thomas15high,hanna2017bootstrapping} give high concentration confidence bounds for off-policy estimators, based on importance sampling methods, which 
suffer from the issue of high variance of IS-based estimators.

Recently, \citet{jiang2020minimax} proposes upper/lower bounds for OPE 
based a minimax formulation of double robustness \citep[e.g.,][]{tang2020doubly,uehara2019minimax}, but their method does not  take the randomness of empirical data  (data uncertainty) into account to get high confidence bounds. 
Concurrently, \citet{duan2020minimax} provides a data-dependent confidence bounds, based on regression-based Fitted Q iteration (FQI), by assuming  the function class is linear and making additional assumptions on MDPs.

Compared with these existing works, our method provides tighter bounds that avoid the  high variance issue of IS, work for behavior-agnostic OPE estimators, and can be directly applied to both continuous and discrete state and action MDPs once a kernel function can be properly defined. 

Finally, the kernel method has been widely used in reinforcement learning, 
but is  
typically used  
 to represent value  functions or transition models ~\citep[e.g.,][]{xu05kernel,xu07kernel,jung2007kernelizing,taylor09kernelized,grunewalder2012modelling,farahmand2016regularized}. %
Our work admits a novel application of kernel in RL, leveraging it as a technical tool for both defining the kernel Bellman loss and constructing high confidence bounds.

\section{Off-Policy Evaluation and Kernel Bellman Statistics}
We start with introducing background on reinforcement learning and off-policy evaluation (OPE), 
then give a review on the kernel Bellman loss  \citep{feng2019kernel} and introduce the concentration bounds for it.
\subsection{Preliminary}
Denote by $M = \langle \mathcal{S}, \mathcal{A}, \mathcal{P}, r, \gamma, \mu_{0} \rangle$  a Markov decision process (MDP), where $\mathcal{S}$ is the state space; $\mathcal{A}$ is the action space; $\cP(s' | s, a)$ is the transition probability; $r(s,a)$ is the average immediate reward;  %
$\mu_{0} \in \Delta(\mathcal{S})$ is the initial state distribution, 
and $\gamma \in (0, 1)$ is the discount factor. We focus on the discounted case ($\gamma < 1$) in the main paper, and discuss the extension to  average reward case ($\gamma =1$) in Appendix \ref{sec:average_case}.

A policy $\pi$ specifies a distribution of actions given states, and $\pi(a|s)$ denotes the probability of selecting $a$ given $s$. 
Given a \emph{target} policy $\pi$, 
we are interested in estimating the expected total discounted reward associated with $\pi$:  
\begin{align}
    \eta^{\pi}:= 
    \lim_{T\to\infty} \mathbb{E}_{\pi}\left[ \sum_{t=0}^T \gamma^t r_t
    \right] =
    \E_{\pi}\left [\sum_{t=0}^{\infty}\gamma^{t}r_t\right ]\,,%
\end{align}
where we start from an initial state $s_{0} \sim \mu_{0}$, and execute the policy $\pi$ in the MDP to generate trajectories; 
here $T$ is the trajectory length. We mainly consider the case with infinite horizon ($T= \infty$) in this work. %

Assume we are given 
an observed set of transition pairs $\mathcal D = (s_i, a_i, r_i, s_i')_{i=1}^n$, 
where  $r_i$ and $s_i'$ 
are the observed local reward and the next state following $(s_i,a_i)$. 
The data is assumed to be \emph{off-policy} in that they are collected from some arbitrary, unknown policy (a.k.a. behavior policy), which is different from the target policy $\pi$ of interest.  
The goal of off-policy evaluation (OPE) is to provide interval estimation (upper and lower bounds) for the expected reward $\eta^\pi$
from the off-policy data $\mathcal D$, without knowing the  underlying behavior policy (i.e., behavior agnostic). 

\paragraph{Value Estimation via Q-Function}
We approach this problem by approximating the expected reward $\eta^\pi$  through the Q-function. 
The Q-function (a.k.a state-action value function) 
$Q^\pi \colon \mathcal{S} \times \mathcal{A} \to \mathbb{R}$ 
of a policy $\pi$ is defined by 
$$Q^{\pi}(s, a):= \E_{\pi}\left [\sum_{t=0}^{\infty}\gamma^{t}r(s_t,a_t) ~\big |~ s_0 = s, a_0 = a\right ]\, ,$$
which equals the expected long-term discounted reward when we 
execute policy $\pi$ starting from $(s,a)$. 

Obviously, $Q^\pi$ and $\eta^\pi$ are related via   
\begin{align}\label{equ:etaq} 
    \eta^{\pi} = \eta(Q^\pi):= \E_{s_0 \sim \mu_0, a_0 \sim \pi(\cdot|s_0)}[Q^{\pi}(s_0, a_0)]\,,
\end{align}
where $s_0$ is drawn from the fixed initial distribution $\mu_0$. 
Therefore, given an empirical approximation $\hat Q$ of $Q^{\pi}$, we can estimate $\eta^{\pi}$ by $\eta(\hat Q)$, which can be 
further approximated by drawing an i.i.d. sample $(s_{i,0}, a_{i,0})_{i=1}^N$ from $\mu_0 \times \pi$: 
\begin{align} \label{equ:xxhatQ}
\eta(\hat Q) \approx 
\hat \eta(\hat Q) := \frac{1}{N}\sum_{i=1}^{N}\hat{Q}(s_{i,0}, a_{i,0})\,.
\end{align}
Since $\mu_0$ and $\pi$ are assumed to be known, 
we can draw an i.i.d.  sample with a very large size $N$ to get arbitrarily high accuracy in the Monte Carlo estimation.

\paragraph{Bellman Equation} 
It is well-known that $Q = Q^{\pi}$ is the unique solution of the \textit{Bellman equation} \citep{puterman94markov}: $Q = \B_{\pi}Q$, 
where $\B_{\pi}$ is the \textit{Bellman evaluation operator}, defined by 
$$
    \B_{\pi}Q(s, a):= 
    \E_{(s',a')}
    \left[r(s, a) + \gamma Q(s', a')~|~ s, a\right ]\,.
$$
where $(s',a')$ is drawn from $s' \sim \mathcal{P}(\cdot ~|~s, a)$ and $a' \sim \pi(\cdot ~|~ s')$. 
For simplicity, we can define the \emph{Bellman residual operator} as $\R_{\pi}Q =  \B_{\pi}Q - Q$.

We can approximate the Bellman operator on a state-action pair $(s_i, a_i)$ by {bootstrapping}: 
$$\hat \B_{\pi}Q(s_i, a_i):= r_i + \gamma \mathbb{E}_{a^\prime\sim \pi(\cdot|s_i^\prime)}[ Q(s_i', a^\prime)]\,,$$
where $r_i, s_i'$ are the observed local reward and the next state of $(s_i,a_i)$, and $a'$ is drawn from policy $\pi$ by Monte Carlo (same as expected Sarsa \citep{sutton98beinforcement}).
Because the policy is known, we can  approximate the expectation term $\E_{a' \sim \pi(\cdot | s'_i)}\left[Q(s_i^{\prime}, a')\right]$ upto arbitrarily high accuracy by drawing a large sample from $\pi(\cdot | s'_i)$.
We can also empirically estimate the Bellman residual by
$\hat \R_\pi Q(s_i, a_i) := \hat \B_\pi Q(s_i, a_i) - Q(s_i, a_i)$,
which provides an unbiased  estimation of the Bellman residual, that is, $\E[\hat \R_\pi Q(s_i, a_i)~|~s_i,a_i~]=\R_\pi Q(s_i,a_i).$ 
Note that $\hat \B_{\pi}Q(s_i, a_i)$ is a function of $(s_i, a_i, r_i, s_i')$, but the dependency on $r_i, s_i'$  is dropped in notation for convenience. 

\qiangremoved{ 
Notice that, instead of the off-policy \emph{value} estimation, here we focus on solving the two key problems we propose for \textit{accountable} off-policy evaluation,
e.g. given $n$ off-policy data $\D=\{s_i, a_i, r_i, s_i'\}_{1 \leq i \leq n}$,
we want to construct a tight interval $[\hat{\eta}^{+}, \hat{\eta}^{-}]$ that contains the expected total discounted reward $\eta^{\pi}$ with high probability.
}

\qiangremoved{ 
\textbf{The Issue of Double Sampling}~~
When the transition of the MDP is stochastic, kernel loss does not suffer from the double sampling problem (i.e. at least two independent sample pair $(r, s')$ for the same state-action pair $x$ to provide consistent and unbiased estimation,
while prior algorithms such as residual gradient \citep{baird95residual} give biased estimation without double samples. \red{this should go earlier. And explained more explicitly (with equations, not words)}\red{check how it is done in feng etal.}
}

\paragraph{The Double Sampling Problem} 
A naive approach for estimating $Q^\pi$ is to minimize the empirical Bellman residual. However, a key barrier is the so-called ``double sampling'' problem. 
To  illustrate,  
consider 
estimating $Q^\pi$ by minimizing the square norm of the empirical Bellman residual: 
$$
\hat L_2(Q) = \frac{1}{n}\sum_{i=1}^n\left [ \left (\hat{\mathcal R}_\pi Q(s_i,a_i) \right)^2\right]. 
$$
Unfortunately, minimizing $\hat L_2(Q)$ does not yield a consistent estimation of $Q^\pi$, because $(\hat{\mathcal R}_\pi Q(s_i,a_i))^2$ does not correctly estimate $(\R_\pi Q (s_i,a_i))^2$ due to the square function outside of the empirical Bellman residual (even though 
$\hat R_\pi Q(s_i,a_i)$ is an unbiased estimation of $
\mathcal R_\pi Q(s_i,a_i)$).  %
One strategy for obtaining an unbiased estimation of the exact square $(\R_\pi Q (s_i,a_i))^2$ is to multiply two copies of $\hat \R_\pi Q (s_i,a_i)$ obtained from two \emph{independent} copies of $r_i, s_i^\prime$ following the same $(s_i, a_i)$,  hence requiring \emph{double sampling} \citep{baird95residual}.    
However, double sampling rarely happens in observed data collected from natural environment,   and is 
impractical to use in real-world problems. 
As a result, many works on 
Q-function estimation resort to temporal difference (TD) based algorithms, which, however, suffer from instability and divergence issues, especially when nonlinear function approximation is used.

\subsection{Kernel Bellman Statistics}

Recently, \citet{feng2019kernel} proposed a kernel Bellman loss to address the double sampling problem in value function estimation. 
This objective function provides a consistent  estimation of the derivation  from a given $Q$ to the true $Q^\pi$, 
and hence allows us to  estimate  $Q^\pi$ 
without requiring double sampling, nor risking the instability or divergence in TD-based algorithms. We now give a brief introduction to the kernel Bellman loss and the related U/V-statistics, and their concentration inequalities 
that are useful in our framework.

For notation, we use $x = (s, a)$ to denote a state-action pair, and $x':= (s',a')$ a successive state-action pair following 
$x = (s, a)$ collected under some behavior policy.
We use $\tau = (s, a, r, s')$ to denote a transition pair.

Following \citet{feng2019kernel}, 
let $K:(\mathcal{S}\times\mathcal{A})\times(\mathcal{S}\times\mathcal{A}) \to \mathbb{R}$ be an \emph{integrally strictly positive definite (ISPD) kernel}. 
The expected kernel Bellman loss is defined by 
\begin{align}
    L_K(Q) := \mathbb{E}_{x, \bar{x} \sim \mu} \left[ \R_\pi Q(x) K(x, \bar{x})  \R_\pi Q(\bar{x})\right]\,, \label{equ:kernel_Loss}
\end{align}
where $\bar x = (\bar s, \bar a)$ is an independent and identical copy of $x$, and $\mu$ is any distribution supported on $\mathcal S \times \mathcal A$. %
Note that $\mu$ can be either the on-policy visitation distribution induced by  $\pi$, or some other valid distributions defined by the observed data (the \emph{off-policy} case). 

As shown in \citet{feng2019kernel}, 
the expected kernel Bellman loss fully characterizes the true $Q^\pi$ function in that 
$L_K(Q) \geq 0$ for any $Q,$ and 
\begin{align*}
    L_K(Q) = 0 && 
\text{if and only if} &&
Q = Q^\pi. 
\end{align*}
Another key property of $L_K(Q)$ is that 
it can be easily estimated and optimized from observed transitions, without requiring double samples.  
Given a set of observed transition pairs $\D = \{\tau_i\}_{i=1}^n$, 
we can use the so-called \emph{V-statistics} to estimate $L_K(Q)$:
\begin{align}
    \hat{L}^{V}_{K}(Q):= \frac{1}{n^{2}}\sum_{i,j=1}^n \ell_{\pi, Q}(\tau_i,
    \tau_j)\,,\label{equ:k_vstats}
\end{align}
where $\ell_{\pi, Q}$ is defined by 
\begin{align}\label{equ:ellQpi}
\ell_{\pi, Q}(\tau_i, \tau_j)  %
:= 
\hat{\R}_{\pi}Q(x_i) K(x_i, x_j) 
\hat{\R}_{\pi}Q(x_j)\,.
\end{align}
Another way is to use the so-called \emph{U-statistics}, which removes the diagonal terms in \eqref{equ:k_vstats}:
\begin{align}
    \hat{L}^{U}_{K}(Q):= \frac{1}{n(n-1)}\sum_{1 \leq i\ne j \leq n} \ell_{\pi,Q}(\tau_i, \tau_j)\,.\label{equ:k_ustats}
\end{align}
We call $\hat{L}^{U}_{K}(Q)$ and $\hat{L}^{V}_{K}(Q)$ the kernel Bellman U/V-statistics. 
Both the U/V-statistics can be shown to give consistent estimation of $L_K(Q)$. The U-statistics $\hat L_K^U(Q)$ is known to be an unbiased estimator of the kernel loss $L_K(Q)$, but can be negative and hence unstable when being used as an objective function in practice. 
In comparison,  the V-statistics $\hat L_K^V(Q)$ is always non-negative, that is, $\hat L_K^V(Q) \geq 0$ for any $Q$, and hence behaves more stably in practice. 
We mainly use the V-statistics in this work.

\subsection{Concentration Bounds of Kernel Bellman Statistics}
We now introduce concentration inequalities for the kernel  Bellman statistics, which form an important tool in our framework. 
The inequalities here are classical results from \citet{hoeffding1963probability}. 

\begin{proposition}[Concentration bound of U-/V-statistics]
Consider a set of i.i.d. random transition pairs $\{\tau_i\}_{i=1}^n$ with $x_i \sim \mu$. %
Assume 
 $$\sup_{\tau,\bar{\tau}}|\ell_{\pi, Q}(\tau, \bar\tau))| \leq \ell_{\max} <\infty,$$
and $n$ is an even number,
then we have the following concentration bound for U-statistics, 
\begin{align*}
    \mathbb{P}\left[|\hat{L}_{K}^U(Q) - L_K(Q)| \geq 2\ell_{\max} \sqrt{\frac{\log \frac{2}{\delta}}{n}}~\right]\leq \delta\,,
\end{align*}
and for V-statistics, 
\begin{align*}
    &\mathbb{P}\left[|\hat{L}_{K}^V(Q) - L_{K}(Q)|\geq
    2 \ell_{\max} \left (\frac{n-1}{n}\sqrt{\frac{\log \frac{2}{\delta}}{n}} + \frac{1}{n}\right) \right] \leq \delta\,.
\end{align*}
\label{thm:non-asymptotic-bound}
\end{proposition}
We include the proof in the Appendix~\ref{sec:u-concentration} for completeness. 
In addition, in the case when $Q = Q^\pi$, it can be shown that the i.i.d. assumption can be removed; see Appendix~\ref{sec:non_iid}.

The key quantity in the concentration inequalities is the upper bound $\ell_{\max}$. The lemma below shows that it can be estimated easily in practice.
\begin{restatable}[]{lemma}{boundedq}
Assume the reward function and the kernel function are bounded, i.e. $\sup_x|r(x)| \leq r_{\max}$, $\sup_{x, \bar x}|K(x, \bar x)|\leq K_{\max}$.
Then we have 
\begin{align*}
    \begin{split} 
         &\sup_{x}|Q^{\pi}(x)|\leq Q_{\max} :=   %
          \frac{r_{\max}}{1-\gamma}\,,
        \nonumber\\
    &\sup_{\tau, \bar{\tau}}|\ell_{\pi, Q^\pi}(\tau, \bar{\tau})| \leq \ell_{\max} :=   \frac{4K_{\max} r_{\max}^2}{(1 - \gamma)^2}\, .
    \end{split}
\end{align*}
\label{lem:empirical-kernel-loss-bound}
\end{restatable}

\section{Accountable Off-Policy Evaluation}
\label{sec:apps}
This section introduces our  approach for accoutable  off-policy evaluation.  
We start with the main {variational framework} for providing confidence bounds in Section \ref{sec:CI-OPE},  
and then discuss post-hoc diagnosis analysis and correction for existing estimators in Section \ref{sec:post-hoc-CI} and  \ref{sec:post-hoc-debias}.

\subsection{Variational Confidence Bounds for Off-Policy Evaluation}
\label{sec:CI-OPE}

\begin{algorithm*}[t]
\caption{Confidence Bounds for Off-Policy Evaluation}
\label{algo:certified}
\begin{algorithmic}
\STATE {\bfseries Input:} Off-policy data $\mathcal{D}= \{x_i, r_i, s_i^\prime\}_{i\leq i \leq n}$; maximum reward $r_{\max}$, discounted factor $\gamma$, positive definite kernel $K$,

~~~~~~~~~~~~ RKHS norm $\rho$, random feature $\Phi(\cdot)$,  failure probability $\delta$, Monte Carlo sample size $N$. %
\STATE Draw $\{x_{i,0}\}_{i=1}^N$ from the initial distribution $\mu_0 \times \pi$; 
decide $\lambda_K$ and $\lambda_\eta$ by concentration inequalities. 
\vspace{.3em}
\STATE Calculate the upper bound  $\hat{\eta}^+$ by solving \eqref{equ:rf_upper}, and the lower bound  $\hat{\eta}^-$ by  \eqref{equ:rf_lower}.
\vspace{.3em}
\STATE {\bfseries Output:} $\hat{\eta}^+, \hat{\eta}^{-}$. 
\end{algorithmic}
\end{algorithm*}
We consider the problem of providing confidence bounds for the expected reward $\eta^\pi$. To simplify the presentation, we focus on the upper bounds here, and one can derive the lower bounds with almost the same arguments. 

Let $\mathcal{F}$ be a function class that contains $Q^\pi$, that is, $Q^\pi \in \mathcal F$. 
Then we can construct an upper bound of $\eta^\pi$ 
by solving the following variational (or functional) optimization on $\mathcal{F}$: 
\begin{align}
    \eta^+ = \max_{Q\in\mathcal{F}}\Big\{  [\eta(Q)], \ \text{s.t.} \ L_{K}(Q) \leq \lambda \Big\}, ~~~~\ \forall \lambda > 0. 
\end{align}
 Here $\eta^+$ is an upper bound of $\eta^\pi$ 
  because $Q^\pi$ satisfies the condition that $Q^\pi \in\mathcal{F}$ and $L_K(Q^\pi) = 0$, and hence ${\eta}^+ \geq \eta^\pi$ follows the definition of the $\max$ operator.

In practice, we can not exactly calculate  $\eta^+$, as it involves the exact kernel loss $L_{K}(Q)$ and the exact $\eta(Q)$. However, if we
 replace  $L_{K}(Q)$ and $\eta(Q)$ with proper empirical estimates, we can derive a computationally tractable high-probability upper bound of $\eta^\pi$.

\begin{proposition}
Let $\hat{\eta}(Q)$ and $\hat{L}_{K}(Q)$ be the estimators of $\eta(Q)$ and $L_{K}(Q)$, such that for $\delta \in (0,1)$,  
\begin{align}
\begin{split} 
     \mathbb{P}\big [\eta(Q^\pi) \leq \hat{\eta}(Q^\pi) + \lambda_{\eta}\big ] & \leq 1-\delta,  \\
     \mathbb{P} \big [ \hat{L}_{K} (Q^\pi) \leq \lambda_{K} \big] & \leq 1-\delta,
    \end{split}
    \label{eq:upper-bound-condition}
\end{align}
where $\lambda_\eta$ and $\lambda_K$ are two constants. Note that \eqref{eq:upper-bound-condition} only needs to hold for $Q = Q^\pi$.
Assume $Q^\pi \in \mathcal F$. Define 
\begin{align}
    \!\!\!\!\!\!\!\hat{\eta}^+ :=  \max_{Q\in\mathcal{F}} \Big \{  [\hat{\eta}(Q)+\lambda_{\eta}], ~~ \text{s.t.} ~~ \hat{L}_{K}(Q) \leq \lambda_{K} \Big\}\,. 
    \label{eq:optimization-problem}
\end{align}
Then we have  $\hat{\eta}^+ \geq \eta^\pi$ with probability $1-2\delta$.
\end{proposition}
\begin{proof}
From the assumption, with probability at least $1-2\delta$, we have both  $\eta(Q^\pi) \leq \hat{\eta}(Q^\pi) + \lambda_{\eta}$ and $\hat{L}_{K} (Q^\pi) \leq \lambda_{K}.$ 
In this case, $Q^\pi$ belongs the feasible set of the optimization in \eqref{eq:optimization-problem}, i.e., 
$Q^\pi\in \mathcal{F}$, $\hat{L}_{K}(Q^\pi)\leq \lambda_{K}$. 
Therefore, we have 
 $\eta^\pi = \eta(Q^\pi) \leq \hat{\eta}(Q^\pi) + \lambda_{\eta} \leq \hat{\eta}^{+}$ by the definition of the $\max$ operator. 
\end{proof}
We can use the kernel Bellman V-statistics  
in Proposition \ref{thm:non-asymptotic-bound} 
to construct $\hat{L}_K(Q^\pi)$ and $\lambda_K$. 
For $\eta(Q^\pi)$, we use the Monte Carlo estimator $\hat \eta (Q^\pi)$ in \eqref{equ:xxhatQ}   
and and set $\lambda_{\eta} =  Q_{\max} \sqrt{2\log(2/\delta)/N}$ by Hoeffding inequality.

\subsection{Optimization in Reproducing Kernel Hilbert Space}
\label{sec:main_rkhs}
To provide tight high confidence bounds,
we need to choose $\mathcal{F}$ to be a function space that is both simple and  rich enough to contain $Q^{\pi}$.
Here we choose $\mathcal{F}$ to be a reproducing kernel Hilbert space  (RKHS) $\mathcal{H}_{\Kf}$ induced by a positive kernel $\Kf(x,\bar x)$.  We should distinguish $\Kf$ with the kernel $K$ used in kernel Bellman loss. 
Using 
RKHS allows us to incorporate a rich, infinite dimensional  set of functions, and still obtain computationally tractable solution for the optimization in \eqref{eq:optimization-problem}.  

\begin{proposition}\label{pro:full_rkhs}
For RKHS $\mathcal{H}_{\Kf}$ with kernel $\Kf$, define  
\begin{align} \label{equ:rkhsf}
\mathcal F = \{Q \in \H_{\Kf} \colon  \norm{Q}_{ \mathcal{H}_{\Kf}}^{2} \leq \rho\},
\end{align}
where $\rho$ is a positive constant.  
Using the $\mathcal F$ in \eqref{equ:rkhsf}, the optimization solution of \eqref{eq:optimization-problem} 
can be expressed as 
\begin{align}\label{equ:Qform}
    Q(\cdot) = \sum_{i=0}^{n} \alpha_{i}f_{i}(\cdot).
\end{align}
 Here 
  $\alpha = [\alpha_i]_{i=0}^n$ are the coefficients decided by the optimization problem, and 
 $f_{0}(\cdot) = \ \E_{x \sim \mu_{0} \times \pi}\left[\Kf(\cdot, x)\right]$,  
\begin{align*}
& f_i(\cdot) = \Kf(\cdot, x_i) - \gamma \E_{a_i^\prime\sim \pi(\cdot|s_i^\prime)}[\Kf(\cdot, x_i^\prime)],     ~~~ i = 1,\ldots, n, 
\end{align*}
where $x_i^\prime = (s_i', a_i')$, with $s_i'$ the observed next state following $x_i=(s_i,a_i)$ and $a_i'$ randomly drawn from $\pi(\cdot | s_i^\prime)$. 

In addition, 
for $Q$ of form \eqref{equ:Qform}, 
$\hat \eta(Q)$ is a linear function of $\alpha$, and both  $\hat L_K(Q)$ and $\norm{Q}_{\H_{\Kf}}^2$ are convex quadratic functions of $\alpha$. 
In particular, we have $\norm{Q}_{\H_{\Kf}}^2 = \alpha^\top B \alpha$, where $B=[B_{ij}]$ is a $(n+1)\times (n+1)$ matrix with $B_{ij} = \langle f_i, f_j\rangle_{\H_{\Kf}}$. 
Therefore, the optimization in \eqref{eq:optimization-problem} reduces to an optimization on $\alpha$ with linear objective and convex quadratic constraints, 
\begin{align*} %
\max_{\alpha^\top B \alpha \leq \rho^2} \Big \{  [c^\top \alpha  + \lambda_{\eta}], %
    ~~~\text{s.t.} ~~~ (\alpha-b)^\top A (\alpha-b)  \leq \lambda_K\Big \}, 
\end{align*} 
where $A,B$ are two $(n+1)\times (n+1)$ positive definite matrices and $b,c$ two $(n+1)\times 1$ vectors, whose definition  
can be found in 
 Appendix \ref{sec:appendix_random_feature}. %
\end{proposition}

\begin{algorithm*}[t]
\caption{Post-hoc Diagnosis for Existing Estimators}
\label{algo:correction}
\begin{algorithmic}
\STATE {\bfseries Input:} Off-policy data $\mathcal{D}= \{x_i, r_i, s_i^\prime\}_{i\leq i \leq n}$; maximum reward $r_{\max}$, discounted factor $\gamma$, positive definite kernel $K$,

~~~~~~~~~~~~  RKHS norm $\rho$, random feature $\Phi(\cdot)$, fail probability $\delta$, Monte Carlo sample size $N$, existing estimation $\widehat{Q}$.
\vspace{.3em}
\STATE Draw $\{x_0^{i}\}_{i=1}^N$ from the initial distribution $\mu_0 \times \pi$; 
Decide $\lambda_\eta$ and $\lambda_K$ by concentration inequalities.

\vspace{.5em}
\STATE Calculate the upper bound  $\hat{\eta}^+$ by solving \eqref{eq:post-rf-upper}, the 
lower bound  $\hat{\eta}^-$ by  \eqref{eq:post-rf-lower}, and 
the debias function  $Q_{\rm{debias}}$ by   \eqref{eq:debias-problem}.

\vspace{.5em}
\STATE {\bfseries Output:} The upper and lower bounds $\hat{\eta}^{+}$, $\hat{\eta}^{-}$, and $Q_{\rm{debias}}(x) = \theta^\top \Phi(x)$. 
\end{algorithmic}
\end{algorithm*}

\paragraph{Random Feature Approximation} 
Unfortunately, 
solving the programming in Proposition~\ref{pro:full_rkhs} requires an  $O((n+1)^3)$ time complexity and 
is too slow when the data size $n$ is very large.  
We can speed up the computation by using random feature approximations. 
The idea is that any positive definite kernel can be expressed as 
$\Kf(x, y) = \int_{\mathcal{W}}\phi(x, w)\phi(y, w)d\mu(w)$, 
where $\phi(x, w)$ denotes a feature map indexed by a parameter $w$ in some space $\mathcal W$ and $\mu$ is a measure on $\mathcal W$. 
A typical example of this  is the random Fourier expansion of stationary kernels by Bochner's theorem  \citep{rahimi2007random}, in which $\phi(x, w) = \cos(w^\top [x,1])$. %

To speed up the computation, 
we draw i.i.d. $\{w_i\}_{i=1}^n$ from $\mu$ and 
take 
\begin{align}\label{equ:dfdkf}
\Kf(x, y) = \frac{1}{m}\sum_{k=1}^{m}\phi(x, w_k)\phi(y, w_k)\,. 
\end{align}
Then one can show that any function $Q$ in the RKHS of  $\Kf$ in \eqref{equ:dfdkf} can be expressed as 
$$
Q(x) = \theta^\top \Phi(x), 
$$
where $\Phi(x) = [\phi(x, w_i)]_{i=1}^m$, with $\norm{Q}_{\mathcal{H}_{\Kf}}^2 = m \norm{\theta}_2^2$, where $\norm{\cdot }_2$ is the typical $L_2$ norm on $\RR^m$. 

 Given the off-policy dataset $\mathcal{D} = \{x_i, r_i, s_i'\}$ with $x_i =(s_i, a_i)$, and an i.i.d. sample 
 $\{x_{i,0}\}_{i= 1}^N$ with $x_{i,0} = (s_{i,0}, a_{i,0})$ 
  from the initial distribution $\mu_{0} \times \pi$. 
The optimization in \eqref{eq:optimization-problem} can be shown to reduce to 
\begin{align}
    \begin{split} 
    &\hat\eta^+= \max_{\|\theta\|_2^{{2}} \leq \rho/m} \Big \{  \left[{{c_0}^\top \theta} + \lambda_{\eta}\right]\,, \\
    &~~~~~~~~~~~~\text{s.t.} ~~ \left(Z\theta - v\right)^\top M \left(Z\theta - v\right) \leq  \lambda_{K}  \Big\}\,,
    \end{split}
    \label{equ:rf_upper}
\end{align}
where 
$v = [r_i]_{i=1}^n \in \mathbb{R}^{n \times 1}$, 
$M \in \mathbb{R}^{n \times n}$ with $M_{ij}   = [K(x_i, x_j)/n^2]_{ij}$, 
 $c_{0}  = \sum_{i=1}^{N} \Phi(x_{i,0})/N \in \mathbb{R}^{m \times 1}$,  and 
\begin{align*}
     Z = \left[~\Phi(x_i) - \gamma  \E_{a_i^\prime \sim \pi(\cdot | s_i^\prime)}\left [ \Phi(x'_i) \right] ~\right]_{i=1}^n  
     \in \mathbb{R}^{n \times m}, 
\end{align*}
with $x_i' = [s_i', a_i']$.
The expectation in $Z$ can be approximated by Monte Carlo sampling from $\pi(\cdot|s_i')$. 

Similarly, we can get the lower confidence bounds via
\begin{align}
    \begin{split} 
    &\hat\eta^-=  \min_{\|\theta\|_2^{{2}} \leq \rho/m} \Big \{  \left[c_0^\top \theta - \lambda_{\eta}\right]\,,\\
&~~~~~~~~~~~~\text{s.t.} ~~ \left(Z\theta - v\right)^\top M \left(Z\theta - v \right) \leq  \lambda_{K}  \Big\}\,.  
    \end{split} 
    \label{equ:rf_lower}
\end{align}
Compared with the programming in Proposition \eqref{pro:full_rkhs}, the optimization problems in \eqref{equ:rf_upper}-\eqref{equ:rf_lower} 
have lower dimensions and is hence much faster when $m \ll n$, since the dimension of $\theta$ is $m$, while that of $\alpha$ is $n+1$.  
We describe the detailed procedure for obtaining the upper and lower confidence bounds in Algorithm \ref{algo:certified}.

\subsection{Post-hoc Confidence Bound of Existing Estimators}
\label{sec:post-hoc-CI}
We extend our method to provide post-hoc confidence bounds 
around existing estimators provided by users.  
Given an existing estimator $\widehat{Q}$ of the Q-function, 
we denote by $Q_{\text{res}}:= Q^\pi - \widehat{Q}$  the difference  between the ground truth $Q^\pi$ and the prior estimate $\widehat{Q}$. 
Assuming $Q_{\text{res}}$ belongs to $\mathcal F$, we obtain an upper bound by  
\begin{align*}
    \max_{Q_{\text{res}}\in\mathcal{F}} \Big \{  [\hat{\eta}(Q_{\text{res}} + \widehat{Q})+\lambda_{\eta}], ~~ \text{s.t.} ~~ \hat{L}_{K}(\widehat{Q} + Q_{\text{res}}) \leq \lambda_{K} \Big\}\,. 
\end{align*}
This can be viewed as a special case of \eqref{eq:optimization-problem} but with the function space anchored around the existing estimator $\widehat Q$. 

Similar to the earlier case, 
in the case of random feature approximation, the optimization reduces to 
\begin{align}
    \hat{\eta}^+ =  &\max_{\|\theta\|_2 ^{2} \leq \rho/m} \Big \{  \left[{\hat{\eta}(\widehat{Q}) + c_0^\top \theta} + \lambda_{\eta} \right]\,,\nonumber\\
    &~~~~~~\text{s.t.} ~~ \left(Z\theta - \zeta\right)^\top M \left(Z\theta - \zeta\right) \leq\lambda_{K}  \Big\}\,,  \label{eq:post-rf-upper}
\end{align}
where $c_0, Z, M$ are defined as before, $\hat{\eta}(\widehat Q) =\sum_{i=1}^{N}\widehat{Q}(x_{i, 0})/N $, and $\zeta \in \mathbb{R}^{n \times 1}$ is defined to be the TD error vector of $\widehat{Q}(x)$  evaluated at the dataset $\mathcal{D}$; that is, $\zeta = [\zeta_i]_{i=1}^n$ with 
$$
\zeta_{i} = r_i + 
\gamma \E_{a_i'\sim \pi(\cdot|s_i')}\widehat{Q}(s_i^\prime, a_i^\prime) - \widehat{Q}(x_i). 
$$  
The post-hoc lower  bound follows a similar form: 
\begin{align}
    \hat{\eta}^- =  &\min_{\|\theta\|_2^{{2}} \leq \rho/m} \Big \{  \left[{\hat{\eta}(\widehat{Q}) + c_0^\top \theta} - \lambda_{\eta}\right]\,,\nonumber\\
    &~~~~~~\text{s.t.} ~~ \left(Z\theta - \zeta\right)^\top M \left(Z\theta - \zeta\right) \leq  \lambda_{K}  \Big\}\,. \label{eq:post-rf-lower}
\end{align}

\newcommand{\penlen}{.205\linewidth}
\newcommand{\cgapline}{-.026\linewidth}

\begin{figure*}[t]
    \centering
     \begin{tabular}{cccc}
        \multicolumn{4}{c}{
        \includegraphics[width=.95\textwidth]{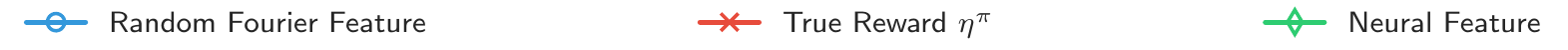}}\\
        \hspace{\cgapline}
        \raisebox{4.1em}{\rotatebox{90}{\small{Rewards}}}\includegraphics[height=\penlen]{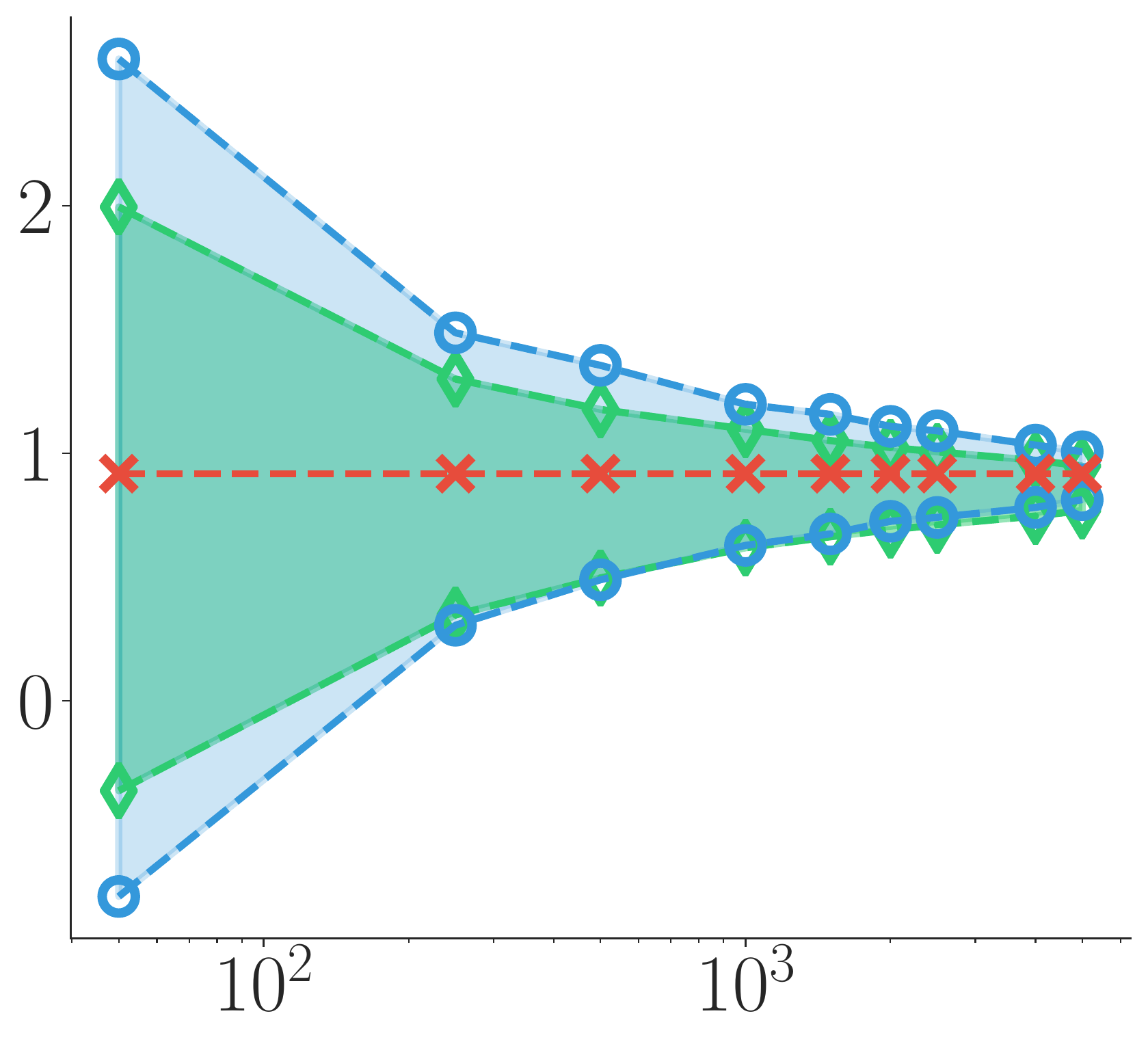}\hspace{-0.1em}&
        \hspace{\cgapline}
        \includegraphics[height=\penlen]{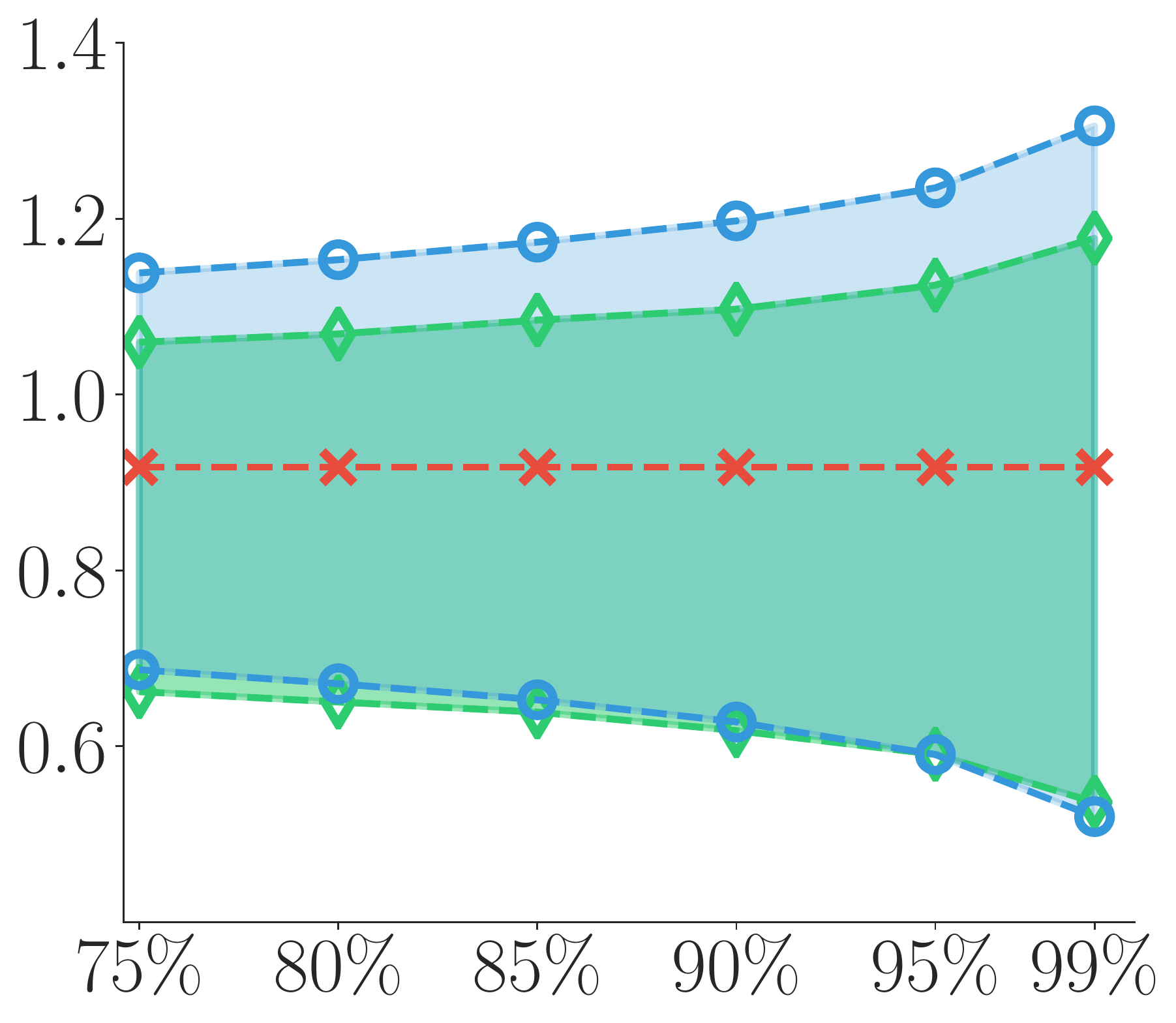}&
        \hspace{\cgapline}
        \includegraphics[height=\penlen]{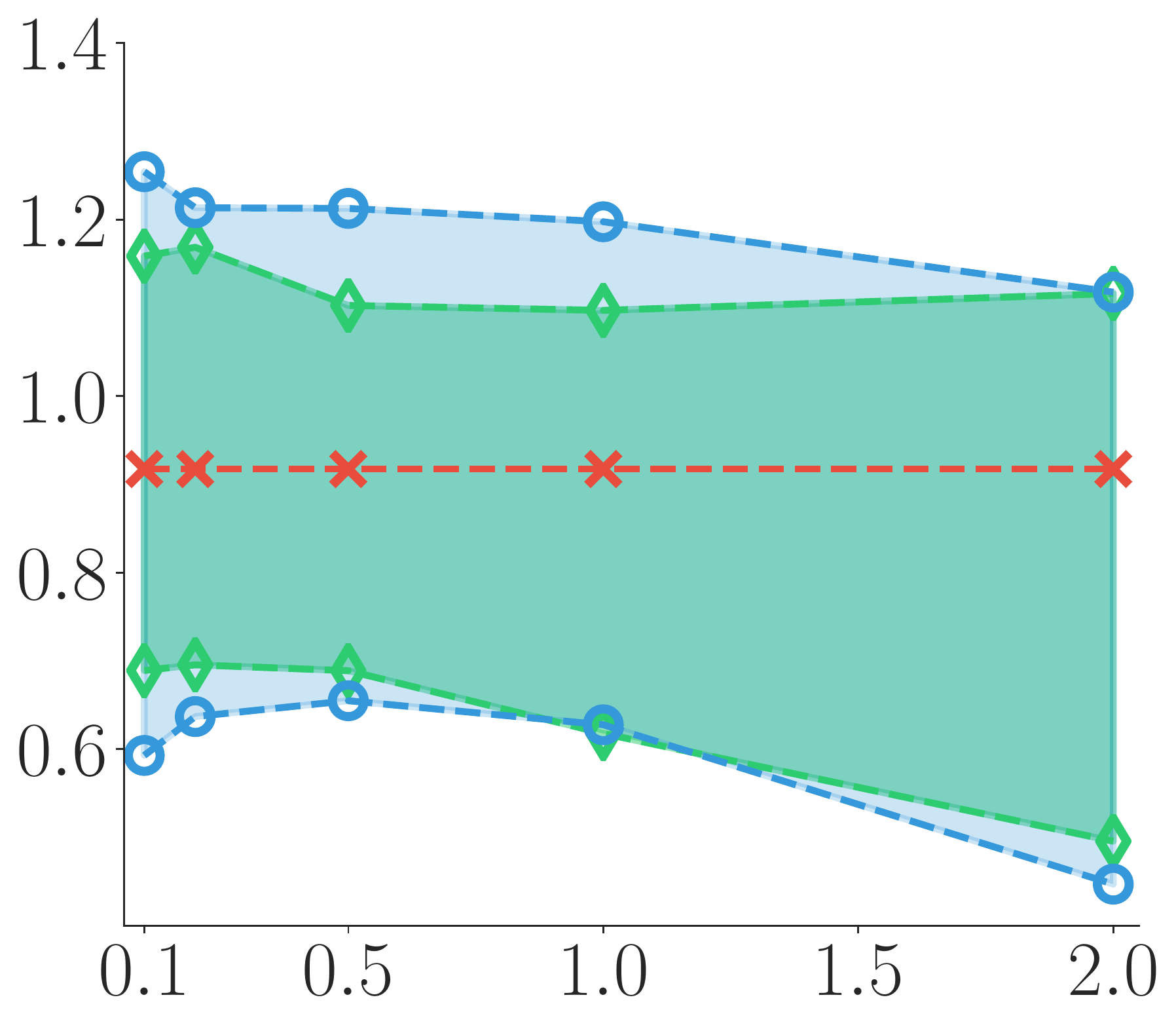}&
        \hspace{\cgapline}
         \includegraphics[height=\penlen]{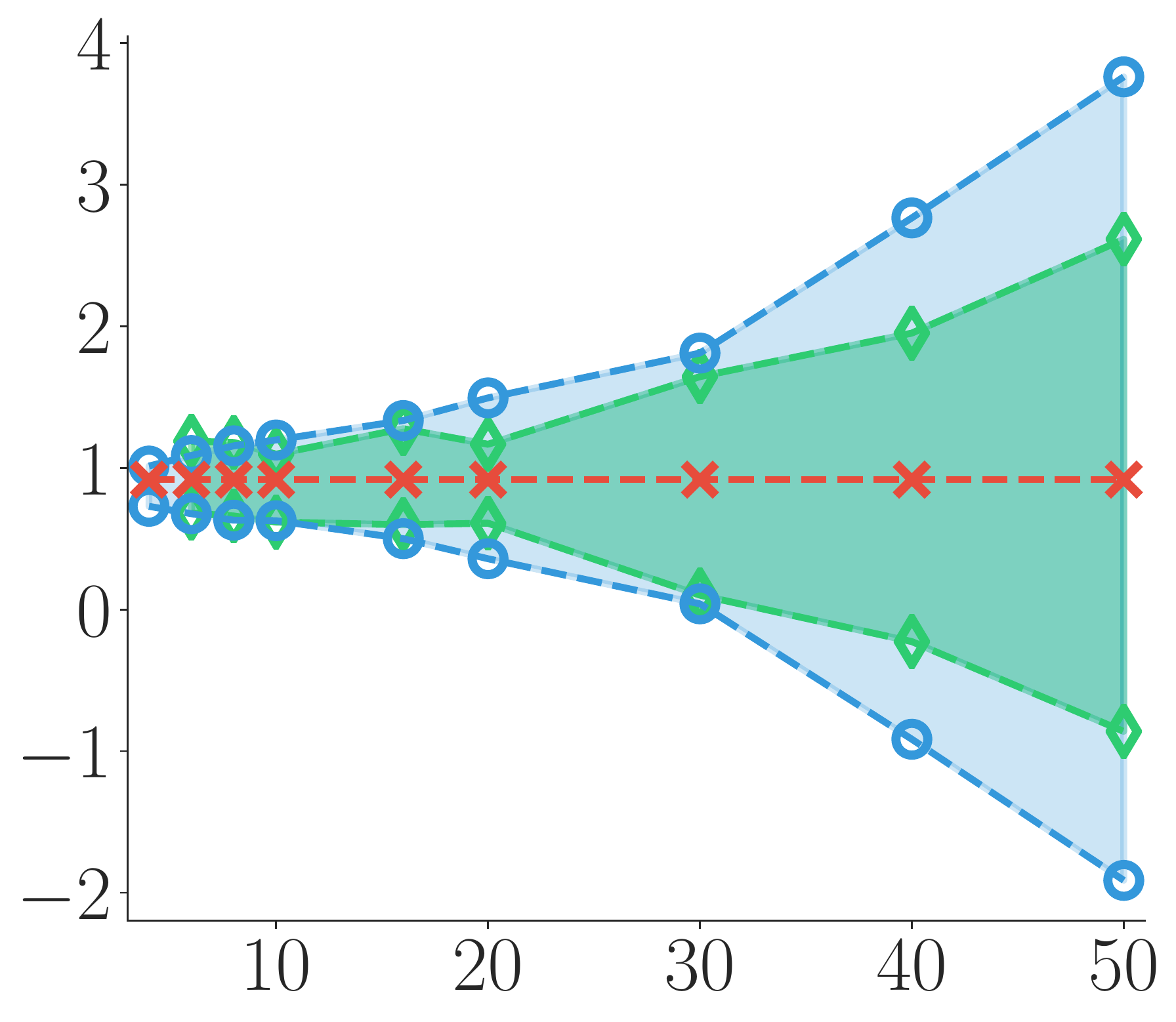}\\
        (a) \small{Number of transitions}& (b) $1 - \delta$ & (c) \small{Temperature} & (d) \small{Number of features} \\
        
        \raisebox{4.1em}{\rotatebox{90}{\small{Rewards}}}\includegraphics[height=\penlen]{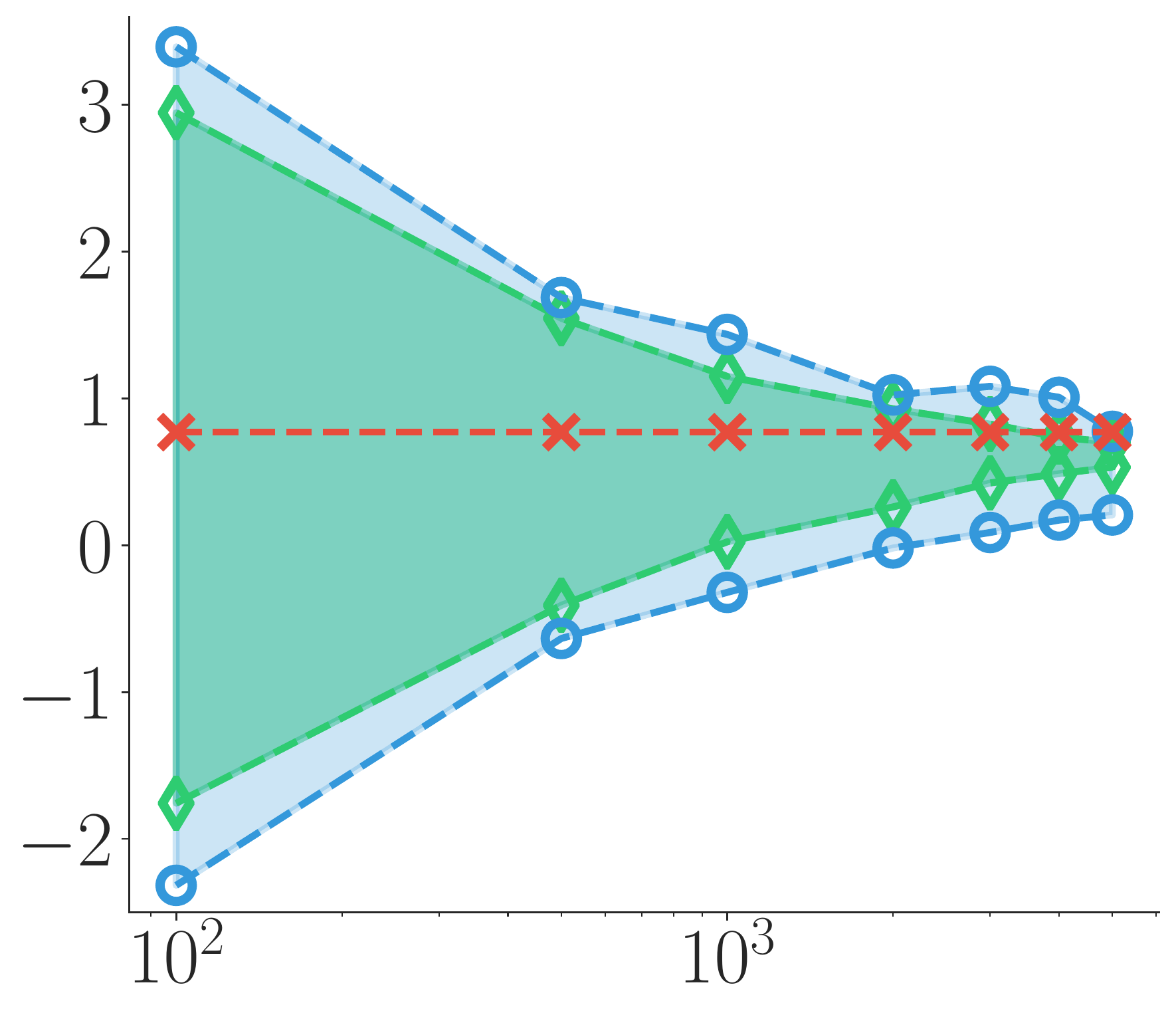}\hspace{-0.1em}&
        \hspace{\cgapline}
        \includegraphics[height=\penlen]{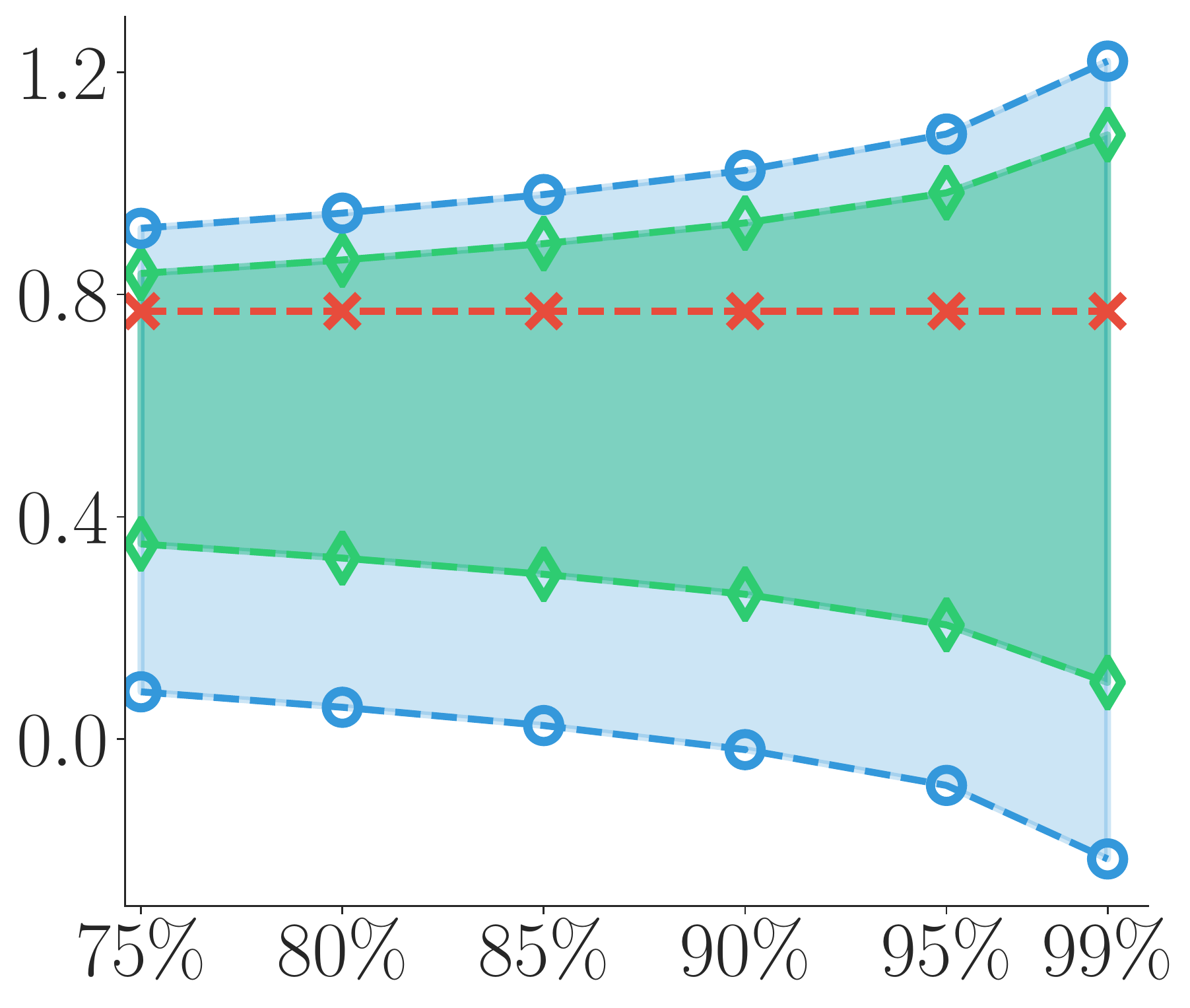}&
        \hspace{\cgapline}
        \includegraphics[height=\penlen]{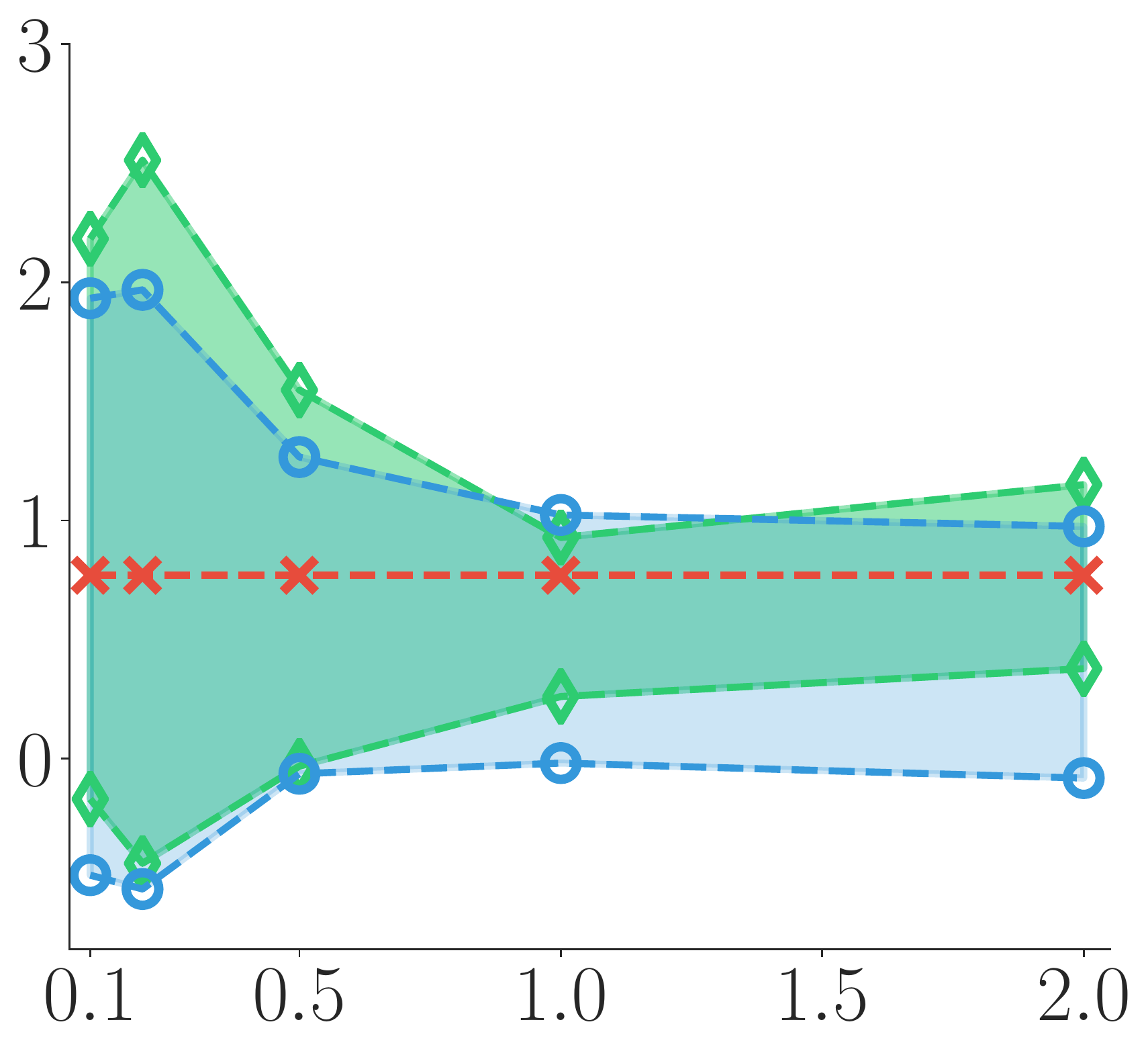}&
        \hspace{-.027\linewidth}
         \includegraphics[height=0.20\linewidth]{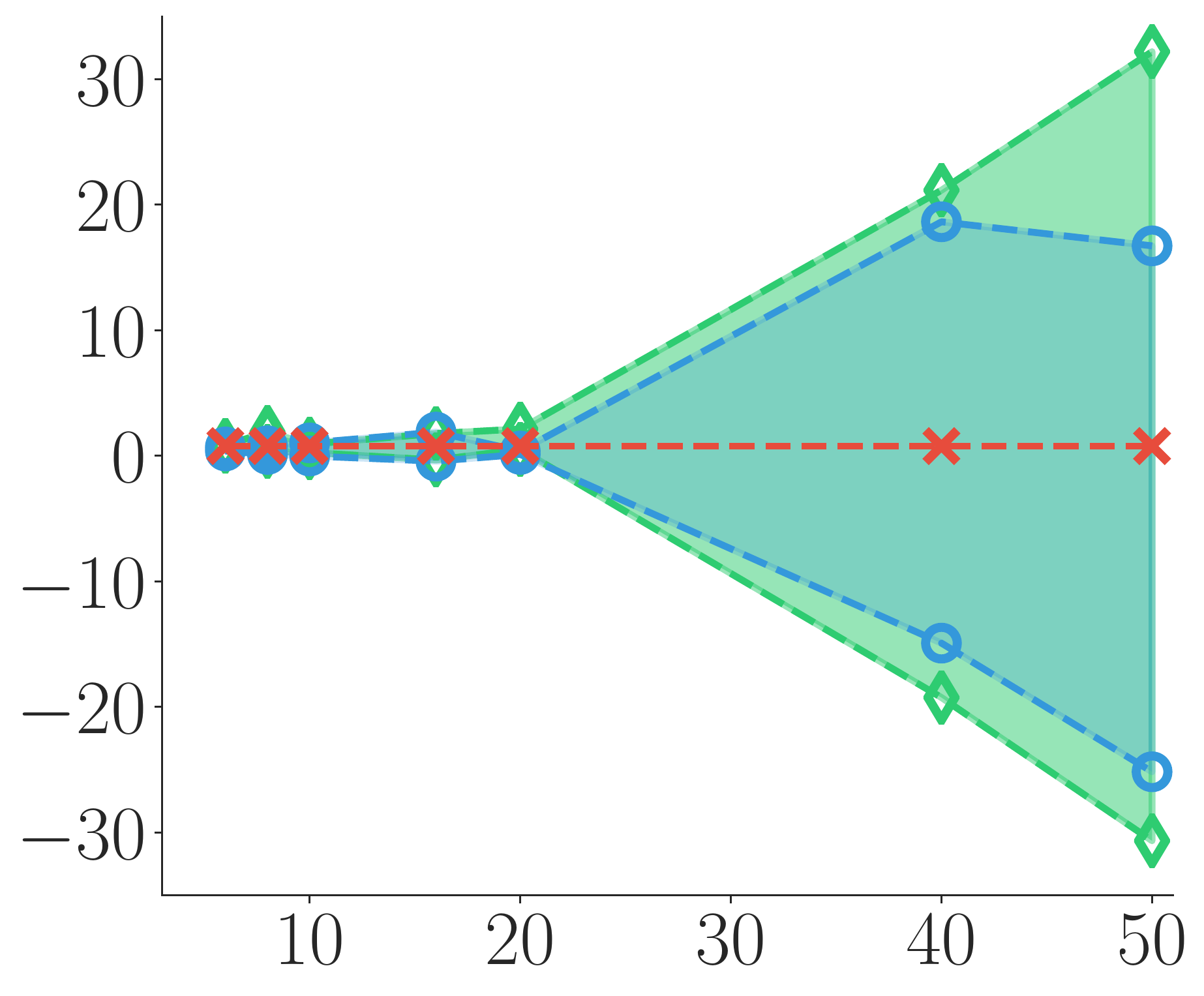}\\
        (e) \small{ Number of transitions}&(f) \small{ $1 - \delta$ }& (g) \small{ Temperature} &(h) \small  Number of features\\
    \end{tabular}
    \caption{ Results on Inverted-Pendulum (a)-(d) and Puck-Mountain (e)-(h).
    The plots show the confidence bounds 
    on the average discounted reward as we 
    vary the number of transitions $n$ in (a) \& (e), failure probability $\delta$  in (b) \& (f), the temperature parameter $\tau$ of the behavior policy in (c) \& (g), and the number of features in (d) \& (h).
    The default parameters (when not varied) are: discounted factor $\gamma=0.95$;  horizon length $T=50$ for Inverted-Pendulum and $T=100$ for Puck-Mountain; number of episodes $20$; failure probability $\delta=0.10$; temperature of the behavior policy $\tau = 1$; and the feature dimension 10.
    }
    \label{fig:pendulum}
\end{figure*}

\subsection{Post-hoc Correction of Existing Estimators}
\label{sec:post-hoc-debias}
In addition to providing confidence bounds around  the existing estimator $\widehat Q$, we may want to further correct $\widehat{Q}$  
when $\widehat{Q}$ is identified to admit large error. 
The post-hoc correction should ensure the corrected estimation falls into the confidence bounds provided earlier, while keeping the information in $\widehat{Q}$ as much as possible.

Our idea is to correct $\widehat{Q}$ by adding a  debiasing term $Q_{\rm{debias}}$, such that $\hat L_{K}^{V}(\widehat{Q} + Q_{\rm{debias}}) \leq \lambda_K$, while keeping $Q_{\rm{debias}}$ as small as possible. This is framed into 
$$
\min_{Q_{\rm{debias}} \in \mathcal F}  \left \{ 
\norm{Q_{\rm{debias}}}_{\H_{\Kf}}^2 ~~\text{s.t.}~~
\hat L_K(Q + Q_{\rm{debias}}) \leq \lambda_K 
\right\}. 
$$
We should distinguish $Q_{\rm{debias}}$ with the $Q_{\rm{res}}$ in Section~\ref{sec:post-hoc-CI}, and $Q_{\rm{res}}$ can not be used for the debiasing purpose because it is designed to give an extreme estimation (for providing the bound), rather than a minimum correction.

In the case of random features approximation, the optimization reduces to
\begin{align}
    &%
    \theta^{\star} =\argmin_{\theta} \Big\{\|\theta\|_2^{2}, \nonumber \\
    &\quad\quad\quad~~~\text{s.t.}~~~ \left(Z\theta - \zeta\right)^\top M\left(Z\theta - \zeta\right) \leq \lambda_{K} \Big\}\,,  \label{eq:debias-problem}
\end{align}
and $Q_{\rm{debias}}(x) = {\theta^{\star}}^{\top}\Phi(x)$. 
If the existing estimator $\widehat{Q}$ already satisfies $\hat{L}_{K}^{V}(\widehat{Q}) \leq \lambda_{K}$,   this provides a zero correction (i.e., $\theta^{\star} = {0}$), since the estimator is already sufficiently accurate. %
This procedure is summarized in Algorithm \ref{algo:correction}.

\newcommand{\delen}{.24\linewidth}
\newcommand{\degapline}{0.0\linewidth}

\begin{figure*}[t]
    \centering
     \begin{tabular}{ccc}
        \multicolumn{3}{c}{
        \includegraphics[width=.85\linewidth]{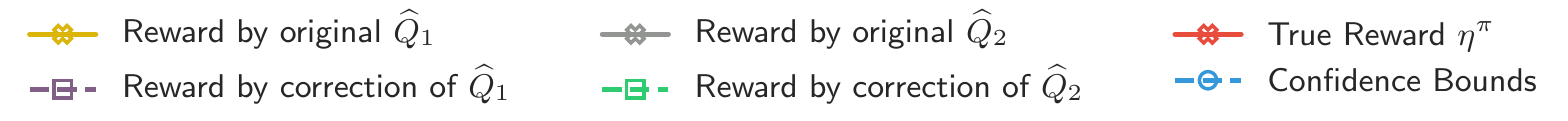}}\\
        \raisebox{5.5em}{\rotatebox{90}{\small{Rewards}}}\includegraphics[height=\delen]{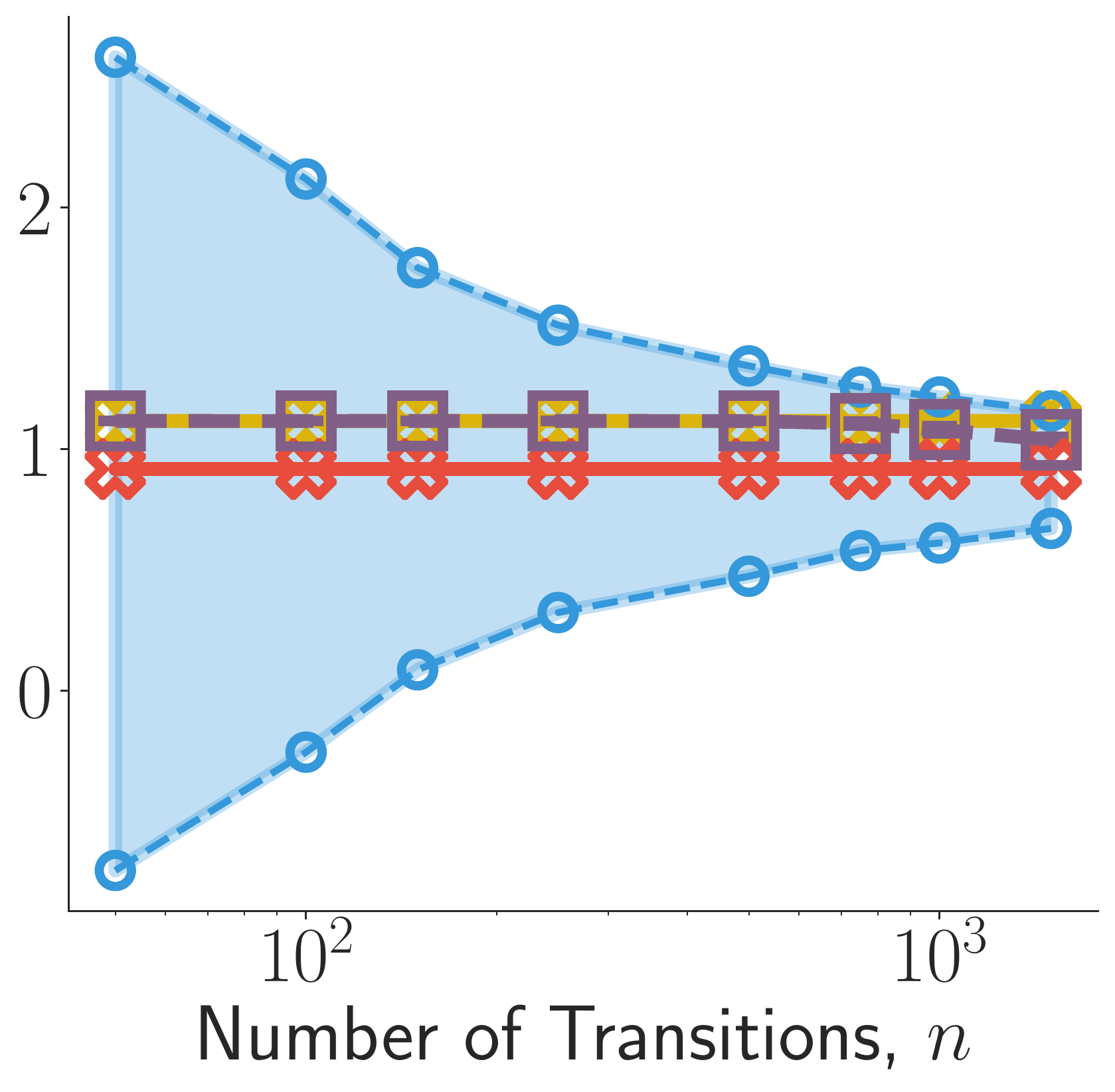}\hspace{-0.1em}&
        \raisebox{5.5em}{\rotatebox{90}{\small{Rewards}}}\includegraphics[height=\delen]{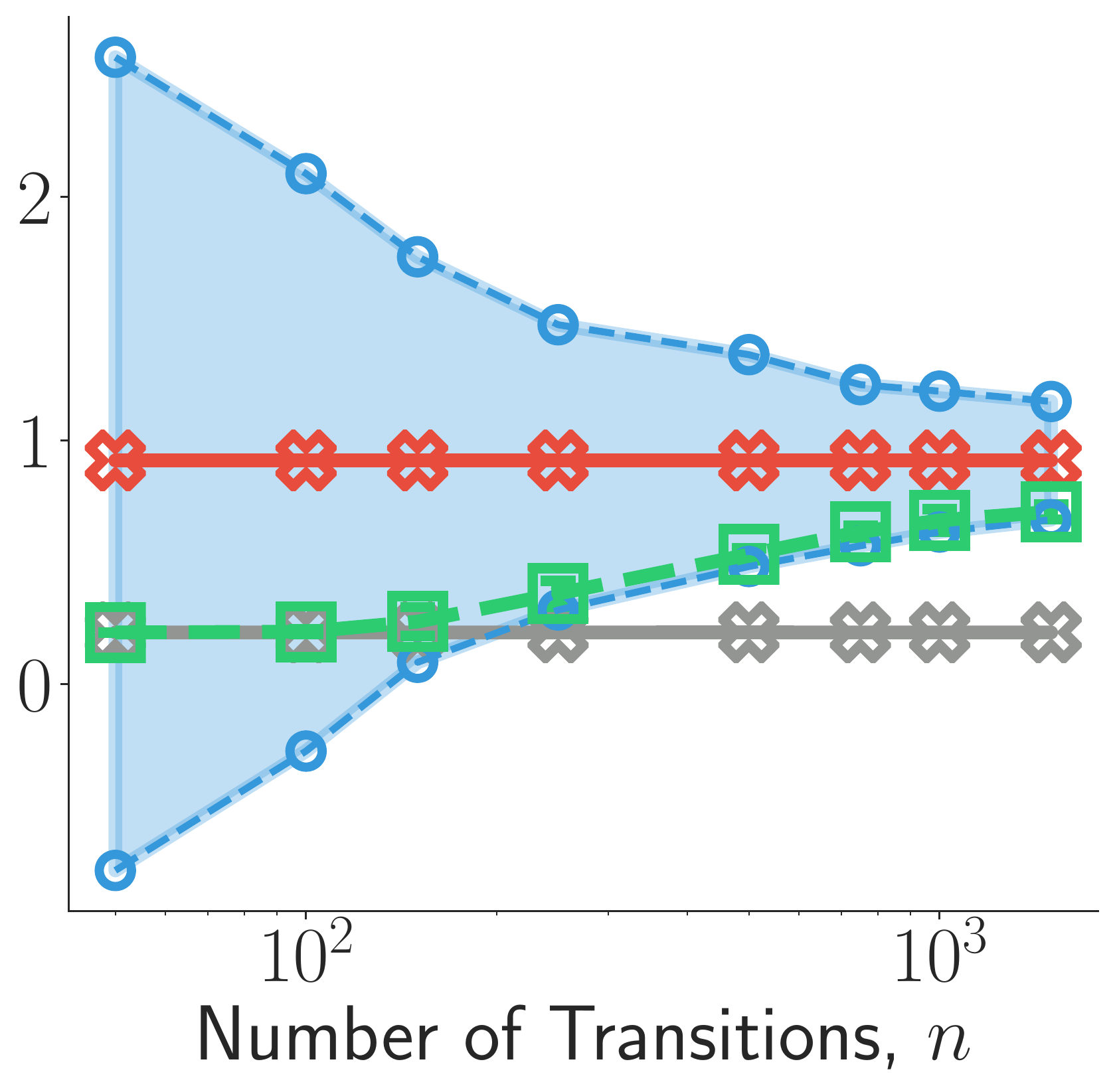} &
        \raisebox{6em}{\rotatebox{90}{\small{Norm }}}\includegraphics[height=\delen]{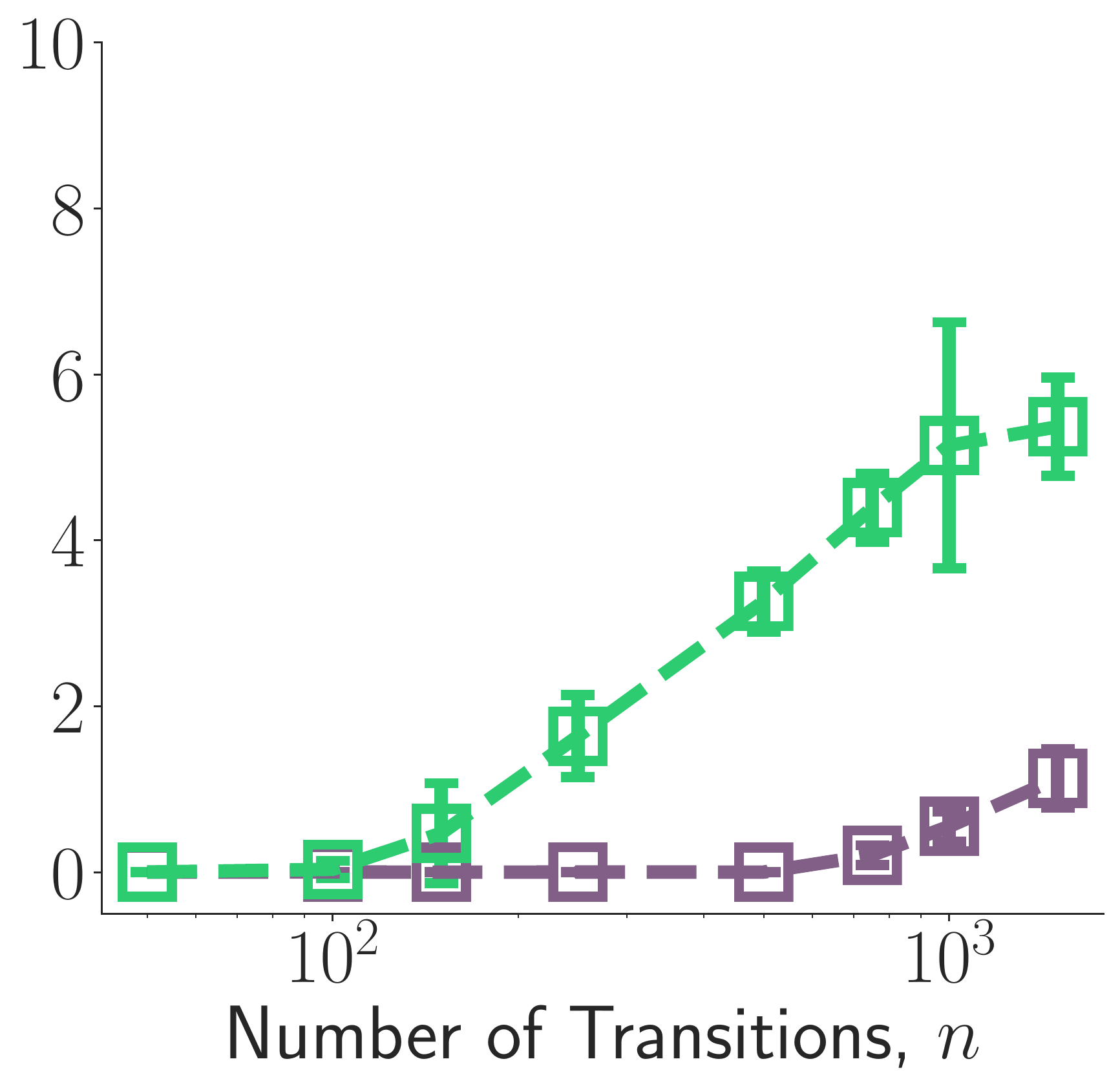}\\
        (a) \small{Diagnosis $\widehat{Q}_1$ with different $n$} & (b) \small{Diagnosis for $\widehat{Q}_2$  with different $n$ } & (c) Norm of  $Q_{\rm{debias}}$\\
        
         \raisebox{5.5em}{\rotatebox{90}{\small{Rewards}}}\includegraphics[height=\delen]{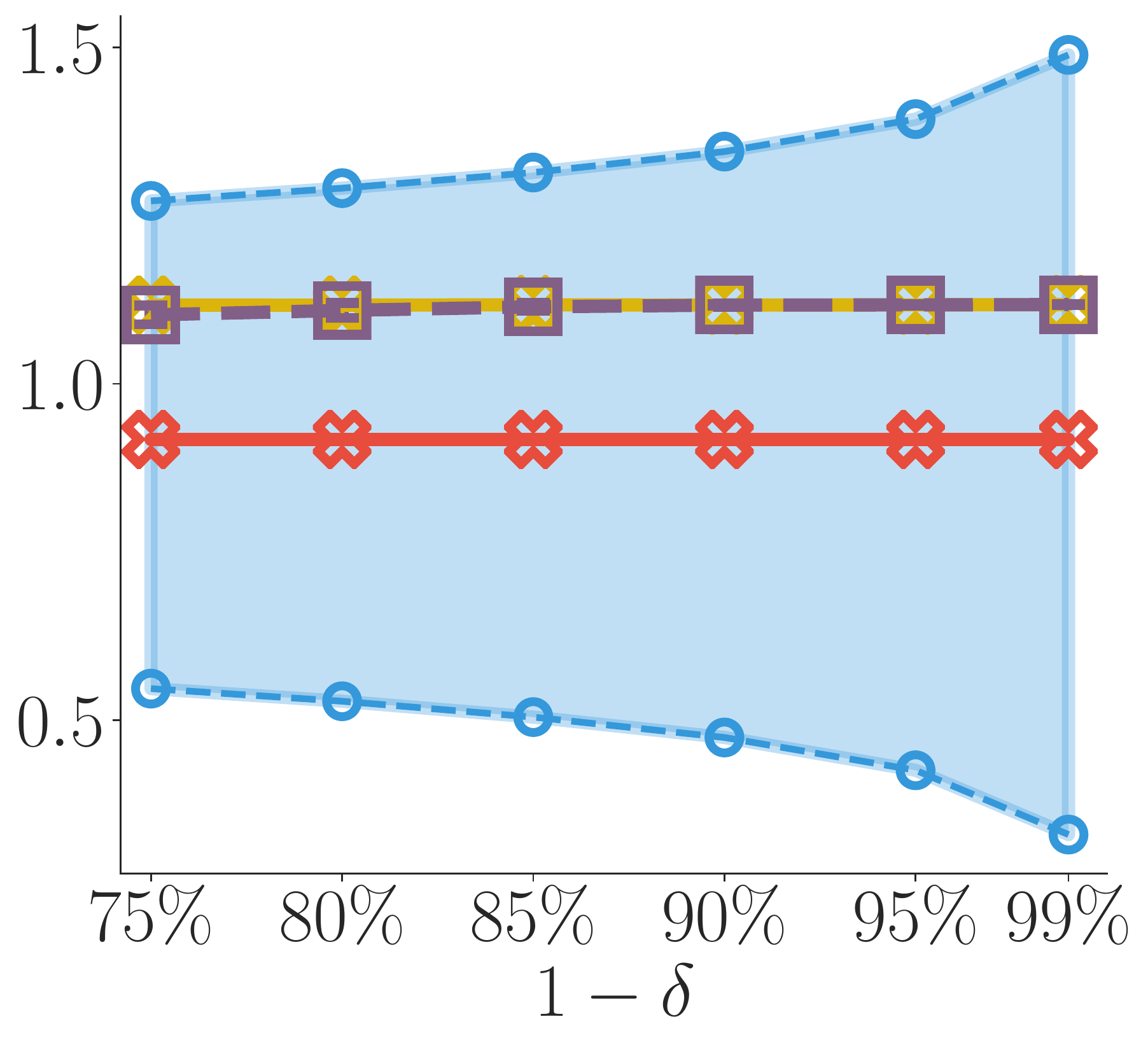}&
        \raisebox{5.5em}{\rotatebox{90}{\small{Rewards}}}\includegraphics[height=\delen]{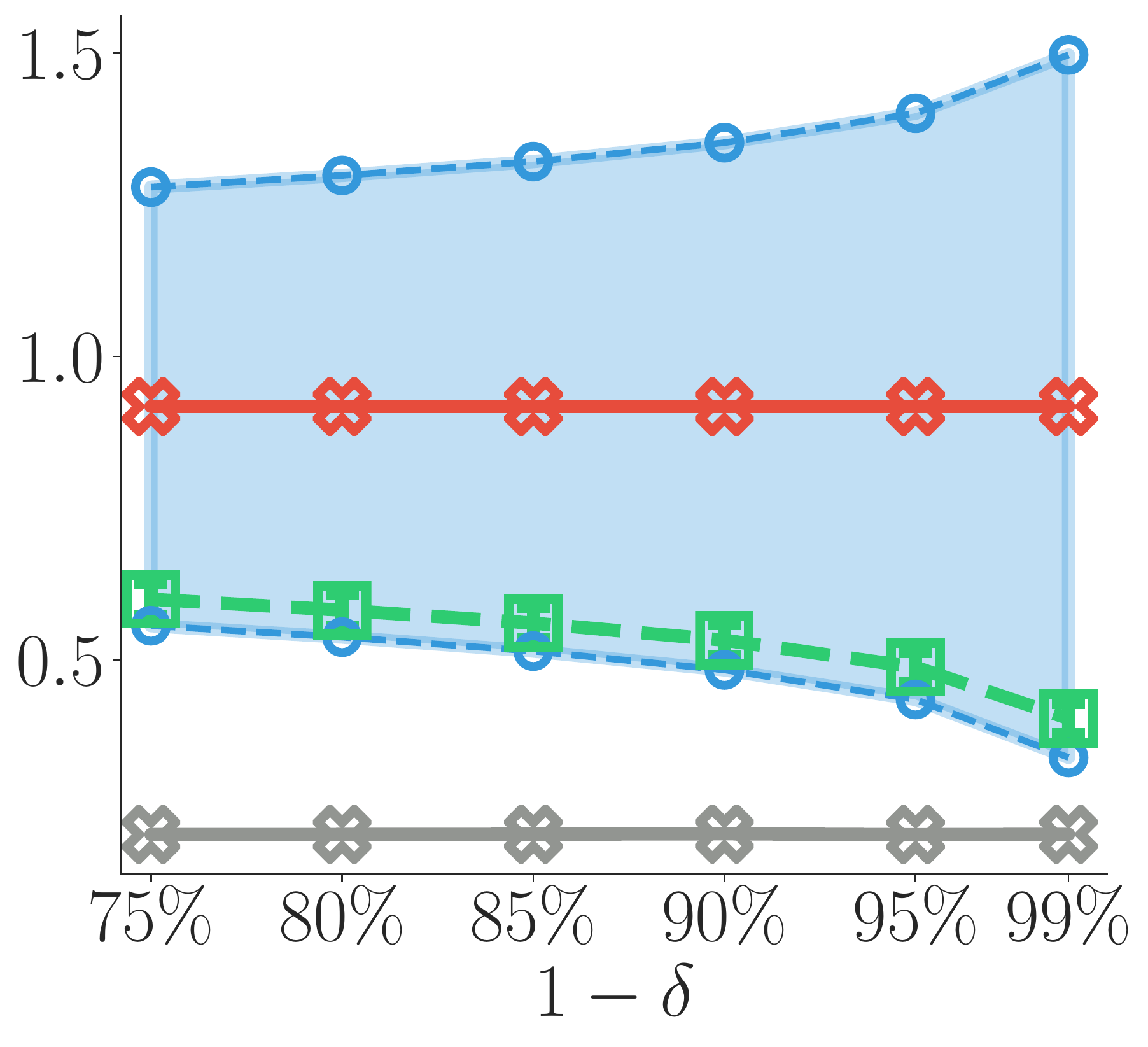}&
         \raisebox{6em}{\rotatebox{90}{\small{Norm }}}\includegraphics[height=\delen]{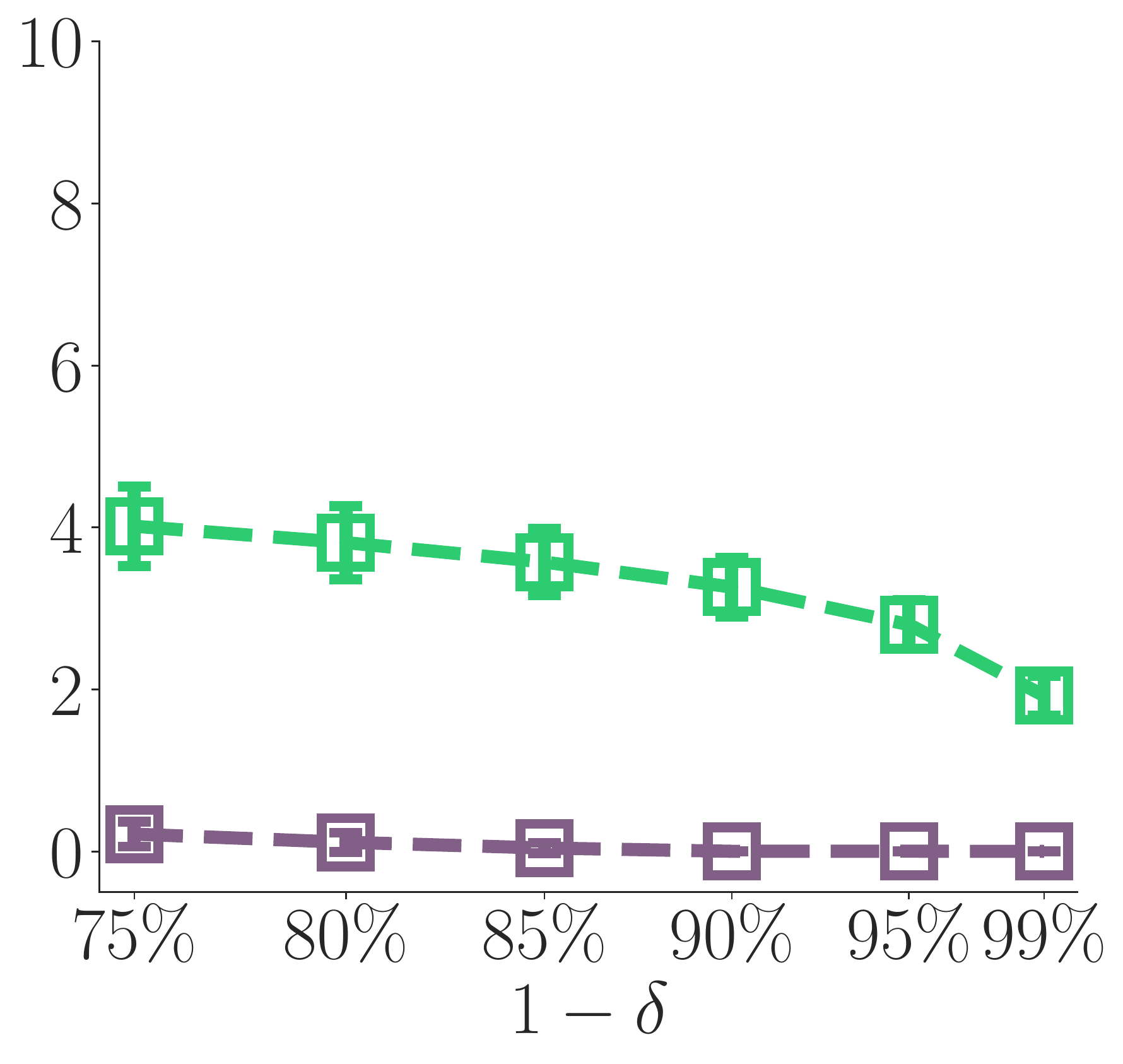}\\
         (d) \small{Diagnosis for $\widehat{Q}_1$ with different $\delta$} & (e) \small{Diagnosis for $\widehat{Q}_2$ with different $\delta$ } & (f) Norm of  $Q_{\rm{debias}}$\\
    \end{tabular}
 
    \caption{Post-hoc diagnosis on Inverted-Pendulum. 
     In figure (a)-(c), we vary the number of transitions with fixed $\delta = 0.10$. In figure (d)-(f), we fixed number of transitions $n=500$ and vary the failure probability $\delta$.  
    The default parameters are: discounted factor $\gamma=0.95$, the horizon length $T=50$, number of transitions $n=500$, failure probability $\delta=0.10$, temperature of the behavior policy $\tau = 1$, and the feature dimension 10.
    }
    \label{fig:pendulum_debias}
\end{figure*}

\section{Experiments}
\label{sec:exp}
We evaluate the proposed algorithms in Section \ref{sec:apps} on two continuous control tasks: Inverted-Pendulum and Puck-Mountain. 
Details of the tasks are in Appendix \ref{sec:exp_appendix}.

For all of our experiments,
we use a Gaussian RBF kernel $K(x_i, x_j) = \exp(-||x_i - x_j||^2_2  / h^{2})$ for evaluating the kernel Bellman statistics,
with a bandwidth selected from a separate batch of training data, and $V$-\textit{statistics} to calculate Equation \eqref{equ:k_vstats} given a set of empirical data.

To construct the behavior and target policies,
we first train an optimal Q-function using deep Q-learning, and use its softmax functions as policies, 
and set the temperature to be higher for the behavior policies to encourage exploration.
Note that our algorithms do not require to know the behavior policies since they are behavior-agnostic.
A description for how to construct policies is in Appendix \ref{sec:exp_appendix}.

\subsection{Confidence Bounds for Off-Policy Evaluation}
We test Algorithm \ref{algo:certified} on the two continuous control tasks, 
Inverted-Pendulum (Figure~\ref{fig:pendulum}(a)-(d)) and Puck-Mountain (Figure~\ref{fig:pendulum}(e)-(h)). 
We solve the convex optimization in Algorithm \ref{algo:certified} using CVXPY \citep{cvxpy,cvxpy_rewriting}, which gives us the upper and lower bound $\hat{\eta}^{+}$ and $\hat{\eta}^{-}$. 
The results are reported w.r.t. the  average discounted reward defined by ${\eta^\pi}/{(1-\gamma)}$. The results are averaged  over 50 random trials.

We use two types of feature maps $\Phi(\cdot)$ to parameterize the state-action value function $Q(x) := \theta^{\top}\Phi(x)$: 

1) Random Fourier features: 
$\Phi(x):= [\cos(\mu_i^\top x+ b_i) ]_{i=1}^m$, where $\mu_{i} \sim \mathcal{N}(0, \frac{1}{h_0^2} \rm{I})$, $b_i \sim \text{Uniform} [0, 2\pi]$, and $h_0$ is a bandwidth parameter.

2) Neural features. We use a small neural network to parameterize $Q$ function, and learn the $Q$ function by minimizing the kernel loss on the training dataset, 
and set $\Phi(\cdot)$ by  selecting a set of neural features  (the neural feature map before the last linear layer) on the validation set.

Figure \ref{fig:pendulum}(a)-(d) demonstrate the evaluation results on Inverted-Pendulum.
From Figure \ref{fig:pendulum}(a) \& (b) we can see that, as we increase the number of transitions, or increase the failure probability $\delta$, 
the confidence bounds become tighter since $\lambda_{K}$ becomes smaller. 
However, it still covers the ground truth $\eta^\pi$, which indicates that our method gives a reasonable confidence  interval for off-policy  estimation.

In Figure \ref{fig:pendulum}(c), we investigate the performance of our method
when we vary a temperature parameter of the behavior policy that controls the concentration of the probability mass. 
The confidence bounds do not change significantly with different
temperature of behavior policies.
In Figure \ref{fig:pendulum}(d), we study the performance of our algorithms with different number of features (including both random Fourier and neural features), 
which shows that we can get a tight confidence interval when the number of features is small. 
This is because when decreasing the number of features,  the Q function is constrained in a lower dimensional function space and hence gives a tighter confidence interval. 
However, decreasing the number of features also increases the risk of model misspecification. 
We also test our method on Puck-Mountain, and report the results in Figure \ref{fig:pendulum} (e)-(h), which show similar results as we observe in Inverted-Pendulum.

\subsection{Post-hoc Diagnosis for Existing Estimators}
We implement Algorithm \ref{algo:correction} to provide post-hoc  diagnosis for existing $Q$ estimators, and test it on  
Inverted-Pendulum (Figure~\ref{fig:pendulum_debias}) and Puck-Mountain (Figure~\ref{fig:puckmountain_debias} in the Appendix \ref{sec:exp_appendix}).
Figure \ref{fig:pendulum_debias} (a)-(f) show
the diagnosis results for two different state-action value function estimators ($\hat{Q}_1$ and $\hat{Q}_2$) on Inverted-Pendulum, which we learn with kernel Bellmen statistics using different number of optimization steps.
Here $\hat{Q}_{1}$ is a relatively accurate estimator while the estimation of $\hat{Q}_{2}$ has larger bias.

 Figure \ref{fig:pendulum_debias} (a)-(c) show that as we increase the number of transitions, the norm of the debias term $Q_{\rm{debias}}$ keeps zero when  
$\hat{Q}_{1}$ and $\hat Q_2$ are inside the confidence interval. %
Meanwhile, when 
$\hat{Q}_{1}$ and 
$\hat{Q}_{2}$ can not provide an accurate estimation (falling outside of the confidence interval), 
our algorithm gives a good post-hoc correction to ensure the corrected estimation are inside the confidence interval. 
As we can see, such post-hoc diagnosis provides 
both confidence bounds and correction for existing estimators, which can be useful in real-world applications.

Figure \ref{fig:pendulum_debias} (d)-(f) demonstrate the performance of our algorithms when we change the failure probability $\delta$. 
Again, the empirical results shows that our algorithms can provide  meaningful post-hoc diagnosis.

We also test our method on Puck-Mountain, following the same procedure as on Inverted-Pendulum, which shows similar results as we find in Inverted-Pendulum.

\section{Conclusion}
In this paper, 
we develop a new variational framework for constructing 
confidence bounds for off-policy evaluation
and extend it to perform post-hoc diagnosis on existing estimators. 
In future work, we will leverage our framework to develop safe and efficient policy optimization and exploration algorithms based on the kernel Bellman statistics.

\section*{Acknowledgment}
This work is supported in part by NSF CRII 1830161 and NSF CAREER 1846421.

\bibliographystyle{icml2020}
\bibliography{references}
\clearpage
\appendix
\onecolumn

\begin{center}
\Large
\textbf{Appendix}
\end{center}

\section{Proof Concentration Bounds}
\label{sec:u-concentration}
The proof of concentrations bounds of U/V-statistics are standard topics in probability and statistics. We provide proof here for completeness. In addition, we show that the concentration bounds still hold in non-i.i.d. cases when $Q = Q^\pi$ due to a special martingale structure from Bellman equation.

\subsection{Proof of Concentration Bound for U-statistics} 
\label{sec:u-concentration-hoeffding}

Assume $X$ is a random variable supported on $\mathcal{X}$. Given some bounded bivariate function\footnote{U-statistics are not limited to the bivariate functions, however, as the kernel loss we discuss in this paper is a bivariate function, we focus on the bivariate function here.} $h:\mathcal{X}^2\to [a, b]$, the U-statistics of $h$ is defined as:
\begin{align*}
    U = \frac{1}{n(n-1)}\sum_{1 \leq i\neq j \leq n} h(X_i, X_j)\,,
\end{align*}
where $X_1, X_2, \cdots, X_n\sim X$ are i.i.d. random variables. It's well known that U-statistics is an unbiased estimation for $\mathbb{E}_{Y, Z\sim X}[h(Y, Z)]$,
and we include the concentration property of U-statistics for completeness.

For simplicity, we assume $n = 2k, k\in\mathbb{Z}$ when we discuss the concentration of U-statistics with Hoeffding's inequality.

\begin{theorem}[Hoeffding's Inequality for U-statistics]
\begin{align*}
    \mathbb{P}\left[\left|U - \mathbb{E}[h] \right|\geq (b-a)\sqrt{\frac{\log \frac{2}{\delta}}{2k}}~\right]\leq \delta \,.
\end{align*}
\end{theorem}
\begin{proof}
The proof is originated from \citet{hoeffding1963probability}, and we restate the original proof here for the completeness.

We first introduce the following notation:
\begin{align}
    V(X_1, X_2, \cdots, X_n) = \frac{1}{k} \sum_{i \in [k]} h(X_{2i-1}, X_{2i})\,.
    \label{equ:definition_V}
\end{align}
It's easy to see that $\mathbb{E}[V] = \mathbb{E}[h]$, and
\begin{align*}
    U = \frac{1}{n!}\sum_{\sigma\in S_n} V(X_{\sigma_1}, X_{\sigma_2}, \cdots, X_{\sigma_n})\,,
\end{align*}
where $S_n$ is the symmetric group of degree $n$ (i.e. we take the summation over all of the permutation of set $[n]$). 

With Chernoff's bound, we can know
\begin{align*}
    \mathbb{P}\left[U \geq \delta\right]\leq\exp(-\lambda \delta) \mathbb{E}[\exp(\lambda U)], \quad \forall \lambda>0\,.
\end{align*}
So we focus on the term $\mathbb{E}\left[\exp(\lambda U)\right]$. 
With Jensen's inequality, we have:
\begin{align*}
    \mathbb{E}[\exp (\lambda U)] = \mathbb{E}\left[\exp\left(\frac{\lambda}{n!}\sum_{\sigma\in S_n} V\left(X_{\sigma_1}, X_{\sigma_2}, \cdots, X_{\sigma_n}\right)\right)\right]\leq \frac{1}{n!}\sum_{\sigma\in S_n} \mathbb{E}\left[\exp(\lambda V(X_{\sigma_1}, X_{\sigma_2}, \cdots, X_{\sigma_n}))\right].
\end{align*}
Thus, 
\begin{align*}
    \mathbb{P}[U - \mathbb{E}[h] \geq \delta] \leq \frac{1}{n!} \sum_{\sigma\in S_n} \mathbb{E} \left[\exp\left(\lambda V(X_{\sigma_1}, X_{\sigma_2}, \cdots, X_{\sigma_n}) - \lambda \mathbb{E}[h] - \lambda \delta\right)\right].
\end{align*}
Notice that, $V$ is sub-Gaussian with variance proxy $\sigma^2 = \frac{(b-a)^2}{4k}$. Thus, with the property of sub-Gaussian random variable, 
\begin{align*}
    \mathbb{E} \left[\exp \left(\lambda V - \lambda \mathbb{E}[h] - \lambda \delta\right)\right] \leq \exp\left(\frac{ \lambda^2(b-a)^2}{8k} - \lambda \delta\right), \quad \forall \lambda > 0.
\end{align*}
When $\lambda = \frac{4k\delta}{(b-a)^2}$,
$\exp(\frac{ \lambda^2(b-a)^2}{8k} - \lambda \delta)$ achieves the minimum $\exp(-\frac{2k\delta^2}{(b-a)^2})$. Thus,
\begin{align*}
    \mathbb{P}\left[U - \mathbb{E}[h] \geq \delta\right]\leq \frac{1}{n!}\sum_{\sigma\in S_n} \exp\left(-\frac{2k\delta^2}{(b-a)^2}\right) = \exp\left(-\frac{2k\delta^2}{(b-a)^2}\right).
\end{align*}
Moreover, with the symmetry of $U$, we have that:
\begin{align*}
    \mathbb{P}[|U - \mathbb{E}[h]| \geq \delta] \leq  2\exp\left(-\frac{2k\delta^2}{(b-a)^2}\right),
\end{align*}
which concludes the proof.
\end{proof}

\subsection{Concentration Bounds for V-statistics}
\label{sec:v-concentration}
We have the following equation for U-statistics and V-statistics
\begin{align*}
    \hat{L}_{K}^V(Q) = \frac{n-1}{n} \hat{L}_{K}^U(Q) + \sum_{i\in [n]} \ell_{\pi, Q}(\tau_i, \tau_i)\,,
\end{align*}
so we can upper bound $|\hat{L}_{K}^V(Q) - L_{K}(Q)|$ via
\begin{align*}
    |\hat{L}_{K}^V(Q) - L_{K}(Q)| \leq \frac{n-1}{n} |\hat{L}_{K}^U(Q)-L_{K}(Q)| + \left|\frac{1}{n^2} \sum_{i\in[n]} \ell_{\pi, Q}(\tau_i, \tau_i) - \frac{1}{n} L_{K}(Q)\right|\,.
\end{align*}
Thus, with the concentration bounds of $\hat{L}_{K}^U(Q)$, and the fact that $|\ell_{\pi, Q}(\tau_i, \tau_i)|\leq \ell_{\max}$, 
and $|L_{K}(Q)|\leq \ell_{\max}$, we have the desired result.
\subsection{Concentration Bounds for Non I.I.D. Samples}
\label{sec:non_iid}
In practice, the dataset $\mathcal{D} = \{x_i, r_i, s_i'\}_{1\leq i \leq n}$ can be obtained with a non i.i.d fasion (such as we collect trajectories in the MDP follow by a policy $\pi$), 
which violates the assumption that samples from $\mathcal{D}$ are drawn independently.
There are concentration bounds for U-statistics with weakly dependent data. For example, \citet{han2018exponential} considers the  concentration of U-statistics when the data are generated from a Markov Chain under mixing conditions. 
Here we show that, in the case when $Q = Q^\pi$, 
the concentration inequality holds without requiring the i.i.d. or any mixing condition, thanks to a martingale structure from the Bellman equation.

\begin{proposition}
Assume the transitions are sampled from the MDP, i.e. $s^\prime \sim\mathcal{P}(\cdot|x)$, $\bar{s}^\prime \sim \mathcal{P}(\cdot|\bar{x})$, then for any joint measure $\nu$ of $(x, \bar{x})$, we have the following property for $Q^\pi$:
\begin{align*}
    \mathbb{E}_{(x, \bar{x})\sim \nu, s^\prime \sim \mathcal{P}(\cdot|x), \bar{s}^\prime\sim \mathcal{P}(\cdot|\bar{x})}[K(x, \bar{x})\cdot \hat{\mathcal{R}}_{\pi} Q^\pi(x) \cdot \hat{\mathcal{R}}_{\pi} Q^\pi(\bar{x})] = 0\,,
\end{align*}
For i.i.d. case, $\nu = \mu\times \mu$.
\label{prop:q_pi_expectation}
\end{proposition}
\begin{proof}
By the definition of Bellman error for $Q$-function, we have that
\begin{align*}
    \mathbb{E}_{s^\prime \sim \mathcal{P}(\cdot|x)} \left[\hat{\mathcal{R}}_\pi Q^\pi(x)\right] = \mathbb{E}_{x^\prime \sim \mathcal{P}(\cdot|x) \times \pi(\cdot | s^\prime)} \left[\hat{\mathcal{R}}_{\pi}Q^{\pi}(x)\right] = 0\,.
\end{align*}
As we can first take expectation w.r.t $s^\prime$ and $\bar{s}^\prime$, we can conclude the proof.
\end{proof}
\begin{theorem}[Concentration Bounds with Non I.I.D. Samples] 
Consider a set of random transition pairs $\{\tau_i\}_{i=1}^n$ with $\tau_i = (s_i, a_i, r_i, s_i')$. 
Assume the state-action pairs $(s_i,a_i)_{i=1}^n$ are drawn from an arbitrary joint distribution, 
and given $(s_i,a_i)_{i=1}^n$, 
the local rewards and next states $(r_i, s_i')_{i=1}^n$ is drawn \emph{independently} from {$r_i = r(s_i, a_i)$} and $s_i' \sim \mathcal P(\cdot ~|~s_i, a_i)$.  

Then $\forall \delta \in (0, 1)$, we have the following concentration inequality for $Q^\pi$:
\begin{align*}
    \mathbb{P}\left[~~\left|\hat{L}_{K}^U(Q^\pi)\right| \geq 2\ell_{\max} \sqrt{\frac{\log \frac{2}{\delta}}{n}}~~\right] \leq \delta\,.
\end{align*}
where $\ell_{\max}$ is given via Lemma \ref{lem:empirical-kernel-loss-bound}.
\end{theorem}
\begin{proof}
First, $\forall ~(x, s^\prime), (\bar{x}, \bar{s}^\prime)$ pair, where $s^\prime \sim \mathcal{P}(\cdot|x)$, $\bar{s}^\prime \sim \mathcal{P}(\cdot|\bar{x})$, from the proof of Proposition \ref{prop:q_pi_expectation}, we can know 
\begin{align*}
    \mathbb{E}_{s^\prime \sim \mathcal{P}(\cdot|x), \bar{s}^\prime \sim \mathcal{P}(\cdot|\bar{x})} [K(x, \bar{x})\cdot \hat{\mathcal{R}}_{\pi} Q^\pi(x) \cdot \hat{\mathcal{R}}_{\pi} Q^\pi(\bar{x})] = 0\,.
\end{align*}
Then we revisit the definition of $V$ in Equation \eqref{equ:definition_V}:
\begin{align*}
    V(X_1, X_2, \cdots, X_n) = \frac{1}{k}\sum_{i \in [k]} h(X_{2i-1}, X_{2i})\,.
\end{align*}
For kernel Bellman statistic, $X_i = (x_i, s_i^\prime)$, where $s_i^\prime \sim \mathcal{P}(\cdot|x_i)$, and $h(X_i, X_j) = K(x_i, x_j)\cdot \hat{\mathcal{R}}_{\pi} Q^\pi(x_i) \cdot \hat{\mathcal{R}}_{\pi} Q^\pi(x_j)$. With Proposition \ref{prop:q_pi_expectation}, we have that 
\begin{align*}
    \mathbb{E}_{x_{2i-1}, x_{2i}, s_{2i-1}^\prime \sim \mathcal{P}(\cdot|x_{2i-1}),s_{2i}^\prime \sim \mathcal{P}(\cdot|x_{2i}) }\left[K(x_{2i - 1}, x_{2i})\cdot \hat{\mathcal{R}}_{\pi} Q^\pi(x_{2i-1}) \cdot \hat{\mathcal{R}}_{\pi} Q^\pi(x_{2i})~\big|~x_1, \cdots, x_{2i - 2}\right] = 0\,,
\end{align*}
as the expectation doesn't depend on how we get $x_{2i-1}$ and $x_{2i}$, but $s_{2i-1}^\prime \sim \mathcal{P}(\cdot|x_{2i-1})$, $s_{2i}^\prime \sim \mathcal{P}(\cdot|x_{2i})$. So we can view $V$ as a summation of bounded martingale differences.

By using the Azuma's inequality for the martingale differences, we can show:
\begin{align*}
    \mathbb{E} \left[\exp(\lambda V)\right] \leq \exp\left(\frac{\lambda^2 (b-a)^2}{8k}\right)\,.
\end{align*}

So the Hoeffding-type bound still holds, following the derivation of Appendix \ref{sec:u-concentration-hoeffding}.
\end{proof}

\textbf{Remark} We have proved that if the environment is Markovian, we still have the desired Hoeffding-type concentration bound for $Q^\pi$, and our algorithms are still valid given non i.i.d. samples.
However, in practice, we still need the data collecting process to be ergodic (which is a general assumption \citep{puterman94markov}), as we want to estimate $Q^\pi$ over all of the $\mathcal{S}\times \mathcal{A}$. 

\textbf{Remark} If we want to consider any $Q$ function other than $Q^\pi$, the non i.i.d $x$ will lead to an additional bias term, which can be difficult to estimate empirically .
We leave this as a future work.

\textbf{Remark} Notice that, here we only use the universal upper bound of $\ell_{\pi, Q}(\tau_i, \tau_i)$, so the bound for $Q^\pi$ is still valid for non i.i.d dataset $\mathcal{D}=\{x_i, r_i, s_i\}_{1\le i \le n}$ if we use Hoeffding-type bound for U-statistics.

\section{Proof of Lemmas}
\boundedq*
\begin{proof}
Recall the definition of $Q^{\pi}$, we have 
$$
    Q^{\pi}(x) = \E_{\pi}\left[\sum_{t=0}^{\infty}\gamma^{t}| s_0 = s, a_0 = a\right] \leq \left(\sum_{t=0}^{\infty} {\gamma}^{t}\right)r_{\max} = \frac{r_{\max}}{1 - \gamma}, ~~ \forall ~x\,.
$$
Recall the definition of $\ell_{\pi, Q^{\pi}}(\tau, \bar{\tau})$ in \eqref{equ:ellQpi}, we have 
\begin{align*}
|\ell_{\pi, Q^{\pi}}(\tau, \bar{\tau})| &= \left|\left(Q(x) - r(x) - \gamma Q(x') \right)K(x, \bar{x}) \left(Q(\bar{x}) - r(\bar{x}) - \gamma Q(\bar{x}') \right)\right| \\
& \leq \sup_{x, \bar{x}}|K(x, \bar{x})|\cdot \sup_{\tau}\left(Q(x) - r(x) - \gamma Q(x')\right)^2 \\
& \leq K_{\max} \cdot \left(\frac{r_{\max}}{1-  \gamma} + r_{\max} + \gamma \cdot \frac{r_{\max}}{1 - \gamma}\right)^2 \\
&= \frac{4K_{\max}r^2_{\max}}{(1 - \gamma)^2}\,.
\end{align*}
\end{proof}

\newpage
\section{Accountable Off-Policy Evaluation for Average Reward  \texorpdfstring{$(\gamma=1)$}{Lg}}
\label{sec:average_case}
In this section, we generalize our methods  to the average reward setting where $\gamma = 1$.
Denote by $M = \langle \mathcal{S}, \mathcal{A}, \mathcal{P}, r, \gamma\rangle$  a Markov decision process (MDP), where $\mathcal{S}$ is the state space; $\mathcal{A}$ is the action space; $\cP(s' | s, a)$ is the transition probability; $r(s,a)$ is the average immediate reward; $\gamma = 1$ (undiscounted case). 
The expected reward of a given policy $\pi$ is 
\begin{align*}
    \eta^\pi = \lim_{T\to\infty} \frac{1}{T+1} \mathbb{E}_{\pi} \left[\sum_{t=0}^T r_t\right]\,.
\end{align*}

In the discounted case ($0 < \gamma < 1$), the value function $Q^{\pi}(s, a)$ is the expected total discounted reward when the initial state $s_0$ is fixed to be $s$, and $a\sim \pi(\cdot |s)$: $Q^{\pi}(s,a)= \E_{\tau \sim \pi_{\pi}}[\sum_{t=0}^{\infty}\gamma^{t}r_t | s_0=s, a_0 =a]$. 
If the Markov process is ergodic \citep{puterman94markov},
the expected average reward does not depend on the initial states.
In the  average case, however, 
$Q^{\pi}(s,a)$ measures the \textit{average adjusted} sum of reward:  
$$
    Q^{\pi}(s, a) := \lim_{T \to \infty}\E_{\pi}\left[\sum_{t=0}^{T}(r_t - \eta^{\pi}) ~\big |~ s_{0} = s, a_0 = a\right]\,,
$$
which is referred to as the \textit{adjusted (state-action) value function}.

Under this definition, 
$(Q^{\pi},\eta^\pi)$  
is the unique fixed-point solution to the following Bellman equation:
\begin{align}
    Q(s, a)= \E_{s' \sim \mathcal{P}(\cdot | s,a), a'\sim \pi(\cdot | s')}[r(s, a) + Q(s', a')  - \eta] \,.
\end{align}

To simplify notation, we still assume $x = (s, a)$, $\bar{x} = (\bar{s}, \bar{a})$, and $x':=(s', a') $ with $s' \sim \mathcal{P}(\cdot| s, a), a' \sim \pi(\cdot | s')$.
Define the \textit{Bellman residual operator} as
$$
    \mathcal{R}_{\pi}Q(x) = \E_{s' \sim \mathcal{P}(\cdot | s,a), a'\sim \pi(\cdot | s')}\left[r(x) + Q(x^\prime) - \eta \right] - Q(x)\,,
$$
where $Q$ is an estimation of adjusted value function, and $\eta$ is an estimation of the expected reward. 
Note that $\mathcal{R}_{\pi}Q(x)$ depends on both $Q$ and $\eta$, even though it is not indicated explicitly on notation.  
Given a $(Q,\eta)$ pair, the kernel loss for $\gamma= 1$ can be defined as 
\begin{align*}
    L_K(Q, \eta) := \mathbb{E}_{x, \bar{x} \sim \mu} [ \mathcal{R}_{\pi} Q(x) \cdot K(x, \bar{x})\cdot \mathcal{R}_{\pi} Q(\bar{x})]\,.
\end{align*}

Given a set of observed transition pairs $\mathcal{D}= \{\tau_i\}_{i=1}^{n}$, and we can estimate $L_{K}(Q, \eta)$ with the following V-statistics:
\begin{align*}
    \hat{L}_{K}^V(Q, \eta) = & \frac{1}{n^2}\sum_{i, j = 1}^n \ell_{\pi, Q}(\tau_i, \tau_j)\,,
\end{align*}
where
\begin{align*}
    \ell_{\pi, Q} (\tau_i, \tau_j) = \hat{\R}_{\pi}Q(x_i) K(x_i, x_j) 
\hat{\R}_{\pi}Q(x_j)\,,
\end{align*}
and
\begin{align*}
    \hat{\mathcal{R}}_{\pi} Q(x_i) = r(x_i) + \mathbb{E}_{a_i^\prime \sim \pi(\cdot|s_i^\prime)}\left[Q(x_i^\prime) \right] - \eta - Q(x_i).
\end{align*}

Similarly, we can estimate $L_{K}(Q, \eta)$ via U-statistics:
$$
\hat{L}_{K}^U(Q, \eta) =  \frac{1}{n(n-1)}\sum_{1\leq i\neq j \leq n}\ell_{\pi, Q}(\tau_i, \tau_j)\, .
$$

\subsection{Concentration Bounds for U-/V-statistics}
Here we  derive the concentration bounds of the U-/V-statistics for the average reward case, under the mild assumption that the reward $r(x)$
and $Q(x)$ are bounded.

\begin{lemma}
Assume the reward function,  the adjusted (state-action) value function and the kernel function are uniformly bounded, i.e. $\sup_{s\in\mathcal{S}, a\in\mathcal{A}}|r(x)|\leq r_{\max}$, $\sup_{s\in\mathcal{S}, a\in\mathcal{A}}|Q^{\pi}(x)|\leq Q_{\max}$, $\sup_{x, \bar{x}} |K(x,\ \bar{x})| \leq K_{\max}$. Then we have
$$
\sup_{\tau, \bar{\tau}}|\ell_{\pi, Q^{\pi}}(\tau, \bar{\tau})| \leq \ell_{\max}:= 4K_{\max}(Q_{\max} + r_{\max} )^2\,.
$$
\end{lemma}
\begin{proof}
By definition, we have 
\begin{align*}
    \sup_{\tau, \bar{\tau}} |\ell_{\pi, Q^{\pi}}(\tau, \bar{\tau})| &\leq | r_{\max} + Q_{\max } - (-r_{\max}) - (-Q_{\max})|^2 \cdot | K(x, \bar{x})| \\ 
    &\leq 4K_{\max} (Q_{\max} + r_{\max})^2.
\end{align*}
    
\end{proof}

With a similar derivation in Appendix \ref{sec:u-concentration}, 
we can have the same Hoeffding-type bounds for U/V-statistics as that of the discounted case, which can be utilized to construct the confidence interval for $\eta^{\pi}$. 

\subsection{Confidence Bounds for Average Reward}
As our final target is to build the confidence interval for the average reward, we follow a similar idea as the discounted case 
to obtain a  high probability upper bound of the expected reward $\eta^\pi$
 by solving the following optimization problem:  
\begin{align*}
    \max_{|\eta|\leq r_{\max}, Q \in \mathcal{F},}\left\{~~\eta \quad\text{s.t.}%
    \quad \hat{L}_{K}^{V}(Q, \eta) \leq \lambda_K\right\}\,,
\end{align*}
where $\eta$ is a scalar variable and $Q\in \mathcal{F}$ is the adjusted value function, which we want to jointly optimize.

\subsection{Optimization in RKHS}
Similar to the discounted case, we can use random feature approximation to speed up the optimization. In this case, the optimization reduces to 
\begin{align*}
    \hat{\eta}^{+} = \max_{|\eta| \leq r_{\max}, \theta}\left\{\eta, ~~~~\text{s.t.}~~~ \left(Z\theta + \eta - v\right)^\top M \left(Z\theta + \eta - v\right) \leq \lambda_{K}\right \}\,.
\end{align*}
where the constants are the same as these defined in Section \ref{sec:main_rkhs}, except that $Z = X - X'$.

\clearpage
\clearpage
\section{Experiments}
\label{sec:exp_appendix}
In this section, we provide the details of the experiments, and some additional experiments for validating the effectiveness of our method.
\subsection{Experimental Details}
\paragraph{Evaluation Environments}
We evaluate the proposed algorithms in Section \ref{sec:apps} on two continuous control tasks: Inverted-Pendulum and Puck-Mountain.

Inverted-Pendulum is a pendulum that has its center of mass  above its pivot point.  It has a continuous state space on  $\mathbb{R}^{4}$. We discrete the action space to be $\{-1, -0.3, -0.2, 0, 0.2, 0.3, 1\}$.
The pendulum is unstable and would fall over without careful balancing.  %
We train a near optimal policy that can make the pendulum balance for a long horizon using deep Q learnings, and use its softmax function as policies.
We set the temperature to be higher for the behavior policies to encourage exploration. 
We use the implementation from OpenAI Gym \citep{brockman2016openai} and change the dynamic by adding some additional zero mean Gaussian noise to the transition dynamic.

Puck-Mountain is an environment similar to Mountain-Car, except that the goal of the task is to push the puck as high as possible in a local valley whose initial position is at the bottom of the valley.
If the ball reaches the top sides of the valley, 
it will hit a roof and change the speed to its opposite direction with half of its original speeds. The reward was determined by the current velocity and height of the puck. The environment has a $\mathbb{R}^{2}$ state space, and a discrete action space with 3 possible actions (pushing left, no pushing, and pushing right).
We also add zero mean Gaussian perturbations to the transition dynamic to make it stochastic.

\paragraph{Policy Construction} 

We use the open source implementation\footnote{https://github.com/openai/baselines.} of deep Q-learning to train a $32 \times 32$ MLP parameterized Q-function to converge.
We then use the softmax policy of the learned Q-function with different temperatures as policies. 
We set a default temperature $\tau = 0.1$ (to make it more deterministic) for the target policy $\pi$.
For behavior policies, we set the temperature $\tau = 1$ as default. 
We also study the performance of our method under behavior policies with different temperatures to demonstrate the  effectiveness of our method under behavior agnostic settings.

\paragraph{Hyperparameters Selection and Neural Feature Learning}
For all of our experiments, we use Gaussian RBF kernel $K(x, \Bar{x}) = \exp\left(-||x - \Bar{x} ||^2_2/h^2\right)$ in the kernel Bellman kernel (e.g.,  for Equation \eqref{equ:k_vstats}). 
We evaluate the kernel Bellman loss on a separate batch of training data, and find that we can set the bandwidth to $h = 0.5$, which will give a good solution.

When we parameterize $Q$ function $Q(x):= \theta^{\top}\Phi(x)$%
with random Fourier features:
$\Phi(x):= [\cos(\mu_i^\top x+ b_i) ]_{i=1}^m$, where $\mu_{i} \sim \mathcal{N}(0, \frac{1}{h_0^2} \rm{I})$, $b_i \sim \text{Uniform} ([0, 2\pi])$, and $h_0$ is a bandwidth parameter.
We select the bandwidth $h_{0}$ from a candidate set $\Pi = \{h_1, h_2, \ldots, h_k\}$ by finding the smallest lower bound and
largest upper bound 
on a separate validation data. 
Specifically, 
for each $h_i \in \Pi$, we follow the procedure of Algorithm \ref{algo:certified} to calculate an upper and lower bounds for $\eta^{\pi}$,
and select the lowest lower bound and the 
largest upper bound as our final bounds. 
Doing this ensures that our bounds are pessimistic and safe. 
In our empirical experiments we set $\Pi = \{0.2, 0.5,0.6, 0.8, 1.0\}$.

Following the similar procedure, we also select a set of   neural features (the neural feature map before the last linear layer) on the validation set, which have  relatively lower kernel loss when we optimize the neural network. 
Similarly, we select two different neural features for each environment and use the pessimisitc upper and lower bounds for all of our experiments.

\paragraph{Constructing Existing Estimators in Post-hoc Diagnosis}
Since we only need to demonstrate the effectiveness of our proposed post-hoc diagnosis process, we simply parameterize Q as a linear function of a small set of random Fourier features ($Q(\cdot) = \theta^{\top}\Phi(\cdot)$), 
and estimate $\theta$ by minimizing the kernel Bellman loss by running gradient descent for different numbers of iterations.
For the experiments in Inverted-Pendulum (Figure~\ref{fig:pendulum_debias}), 
$\hat Q_1$ (resp. $\hat Q_2$) are obtained 
when we run a large (resp. small) number of iterations in training, 
so that $\hat Q_1$ is relatively accurate  while $\hat Q_2$ is not. 
For Puck-Mountain in Figure~\ref{fig:puckmountain_debias},  
the error of both $\hat Q_1$ and $\hat Q_2$ are relatively large.

\begin{figure*}[t]
    \centering
     \begin{tabular}{ccc}
        \multicolumn{3}{c}{
        \includegraphics[width=.85\linewidth]{figs/pendulum_debias/debias_legend.pdf}}\\
        \raisebox{5.5em}{\rotatebox{90}{\small{Rewards}}}\includegraphics[height=\delen]{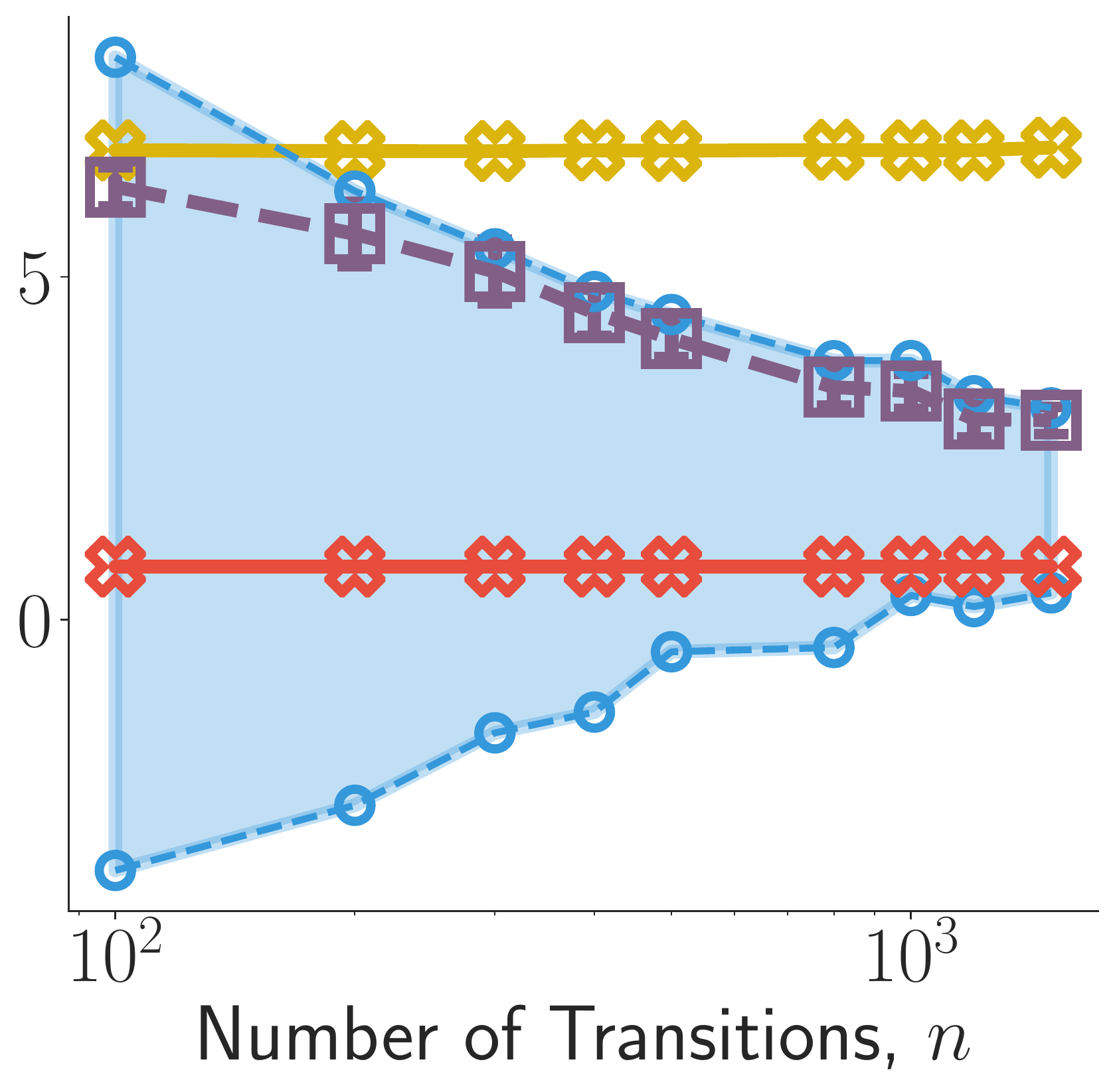}\hspace{-0.1em}&
        \raisebox{5.5em}{\rotatebox{90}{\small{Rewards}}}\includegraphics[height=\delen]{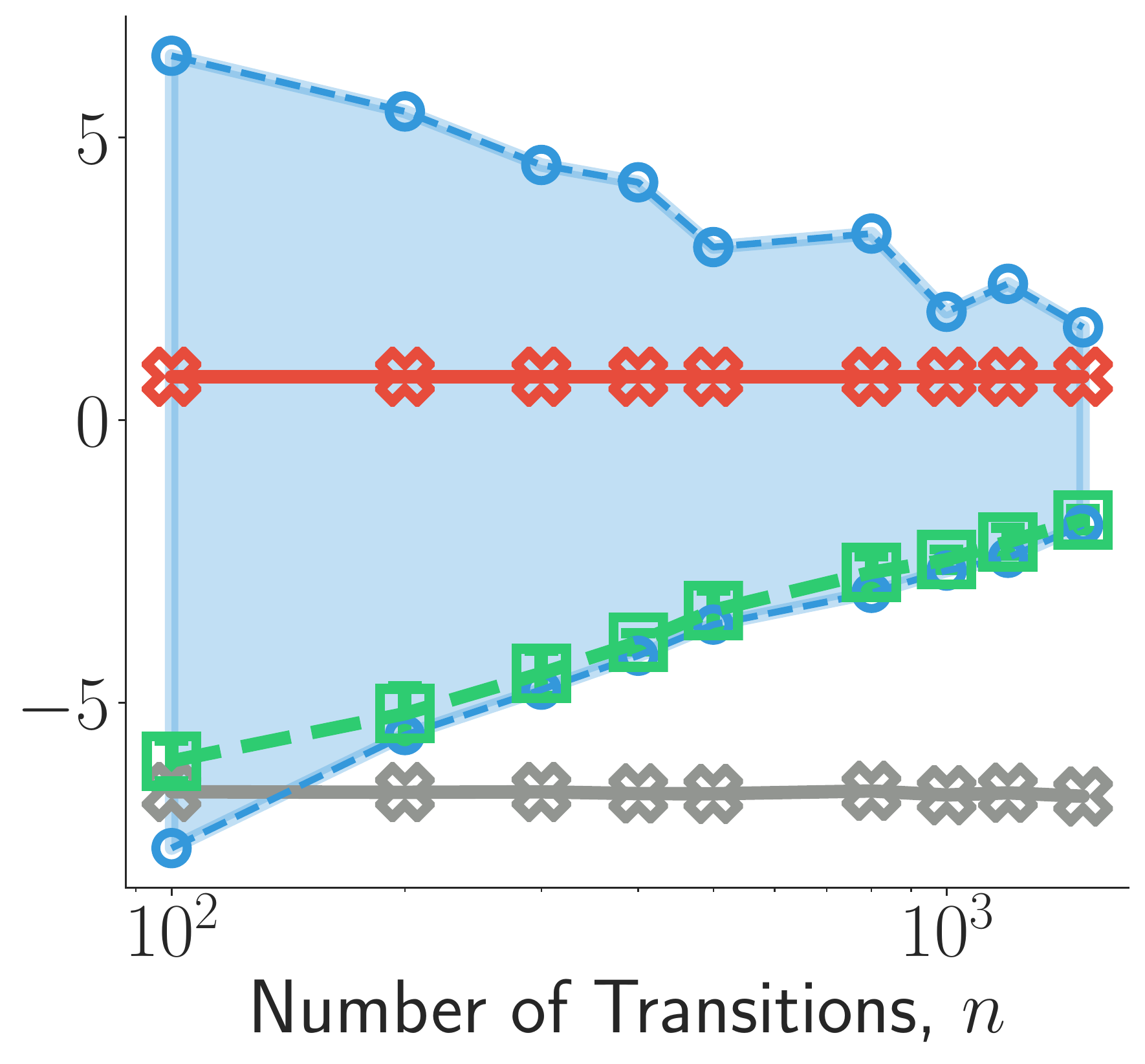} &
        \raisebox{6em}{\rotatebox{90}{\small{Norm }}}\includegraphics[height=\delen]{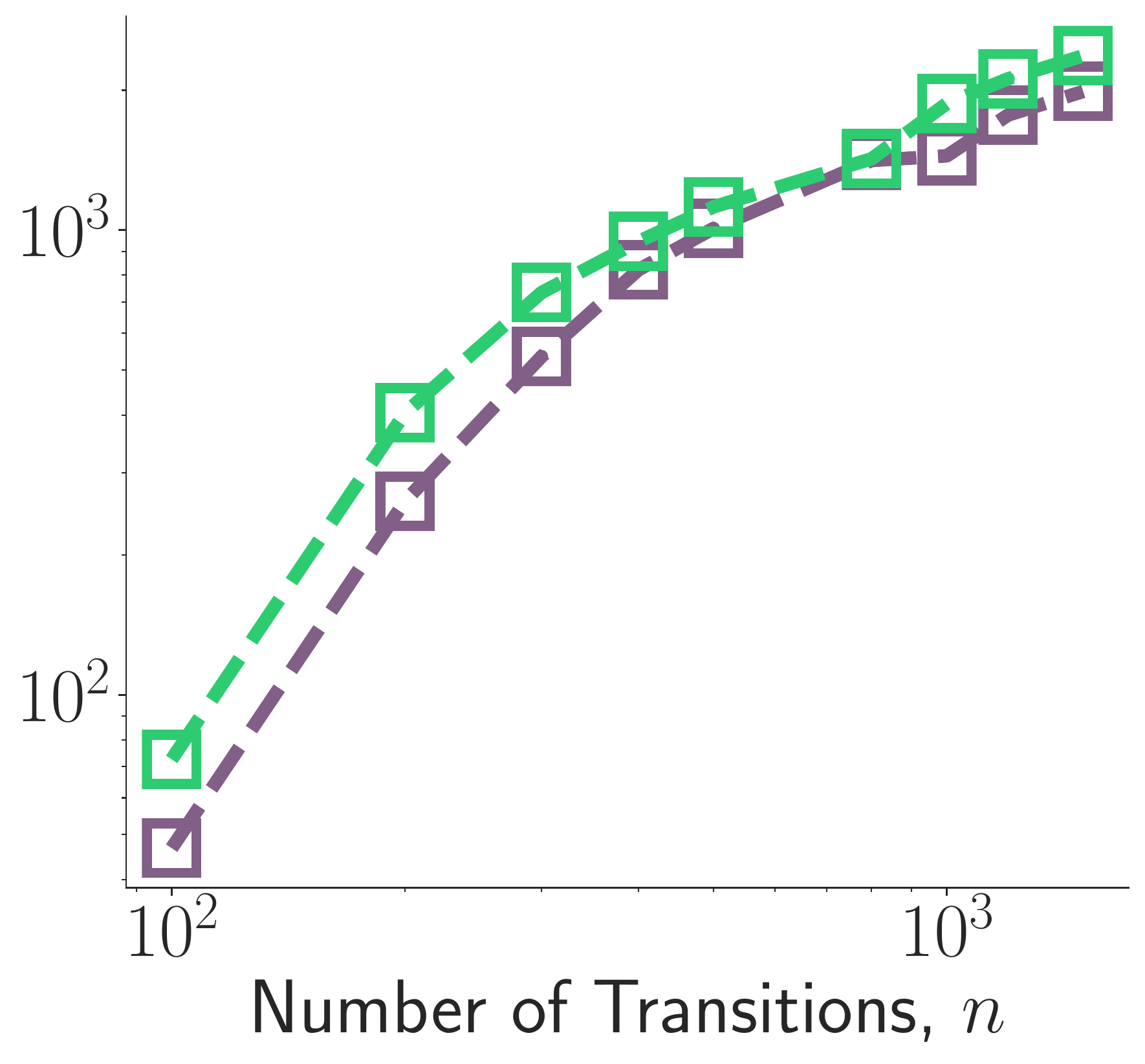}\\
        (a) \small{Diagnosis $\hat Q_1$ with different $n$ } & (b) \small{Diagnosis for $\hat Q_2$  with different $n$ } & (c) Norm of  $Q_{\rm{debias}}$\\
         \raisebox{5.5em}{\rotatebox{90}{\small{Rewards}}}\includegraphics[height=\delen]{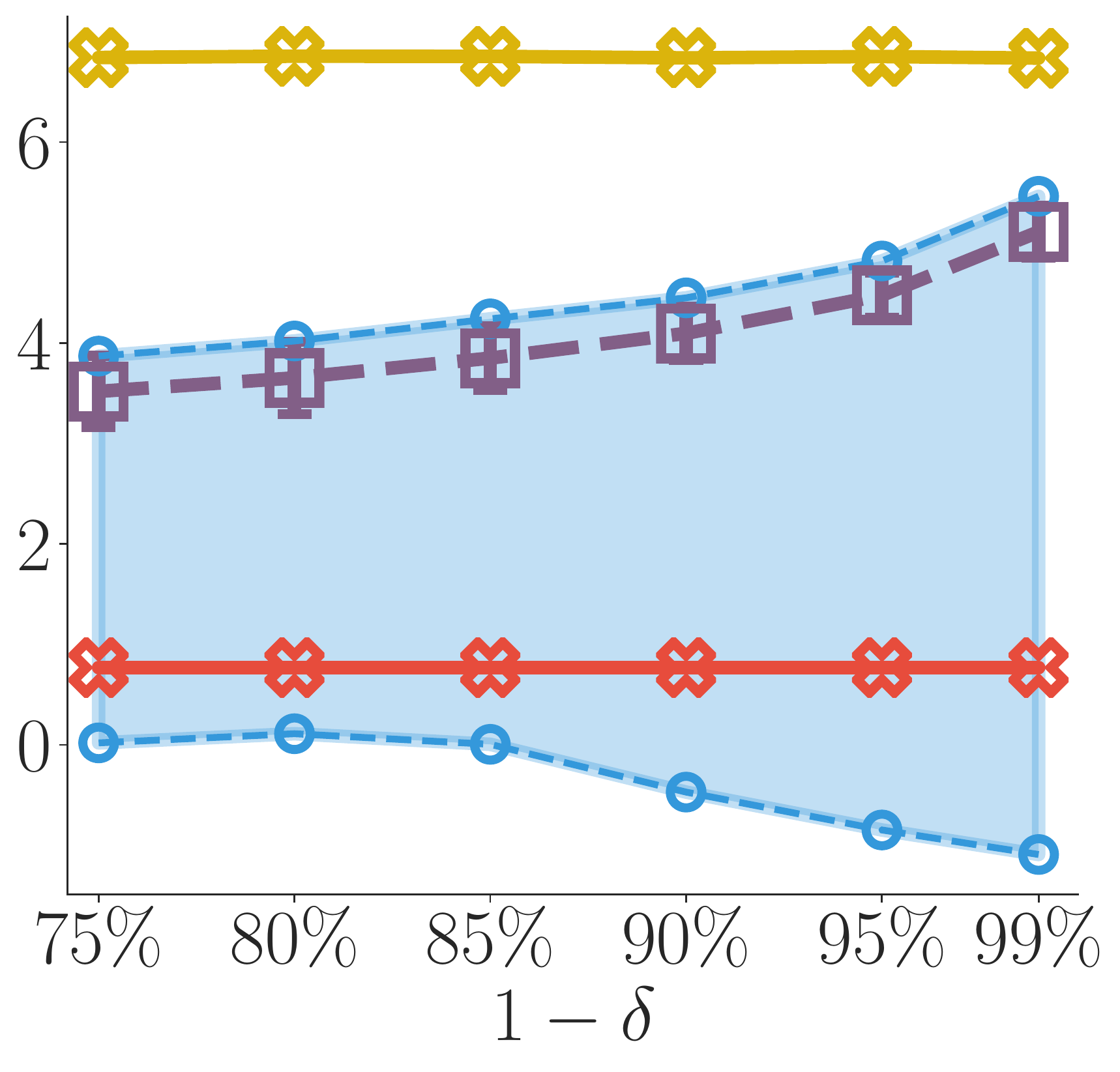}&
        \raisebox{5.5em}{\rotatebox{90}{\small{Rewards}}}\includegraphics[height=\delen]{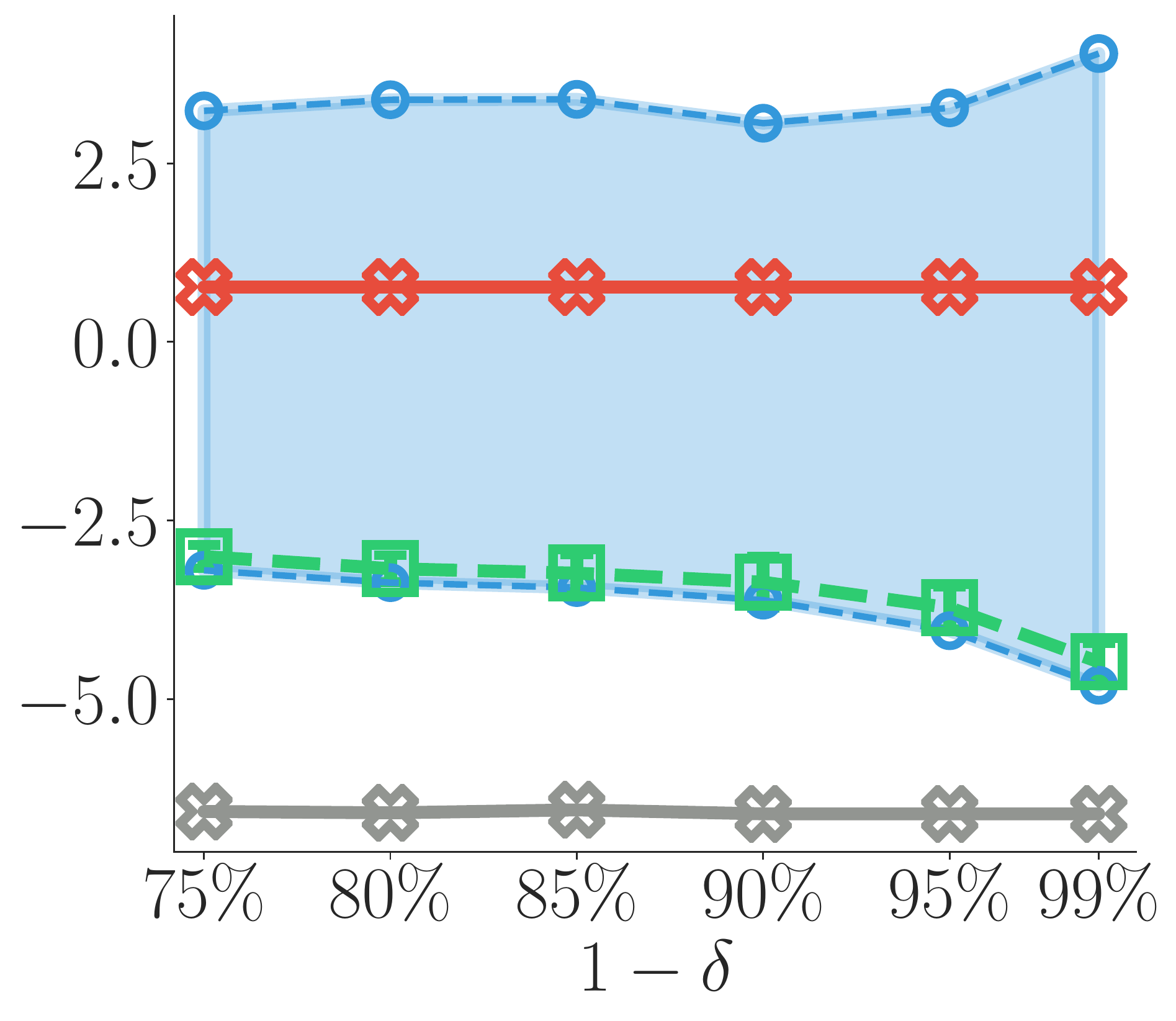}&
         \raisebox{6em}{\rotatebox{90}{\small{Norm }}}\includegraphics[height=\delen]{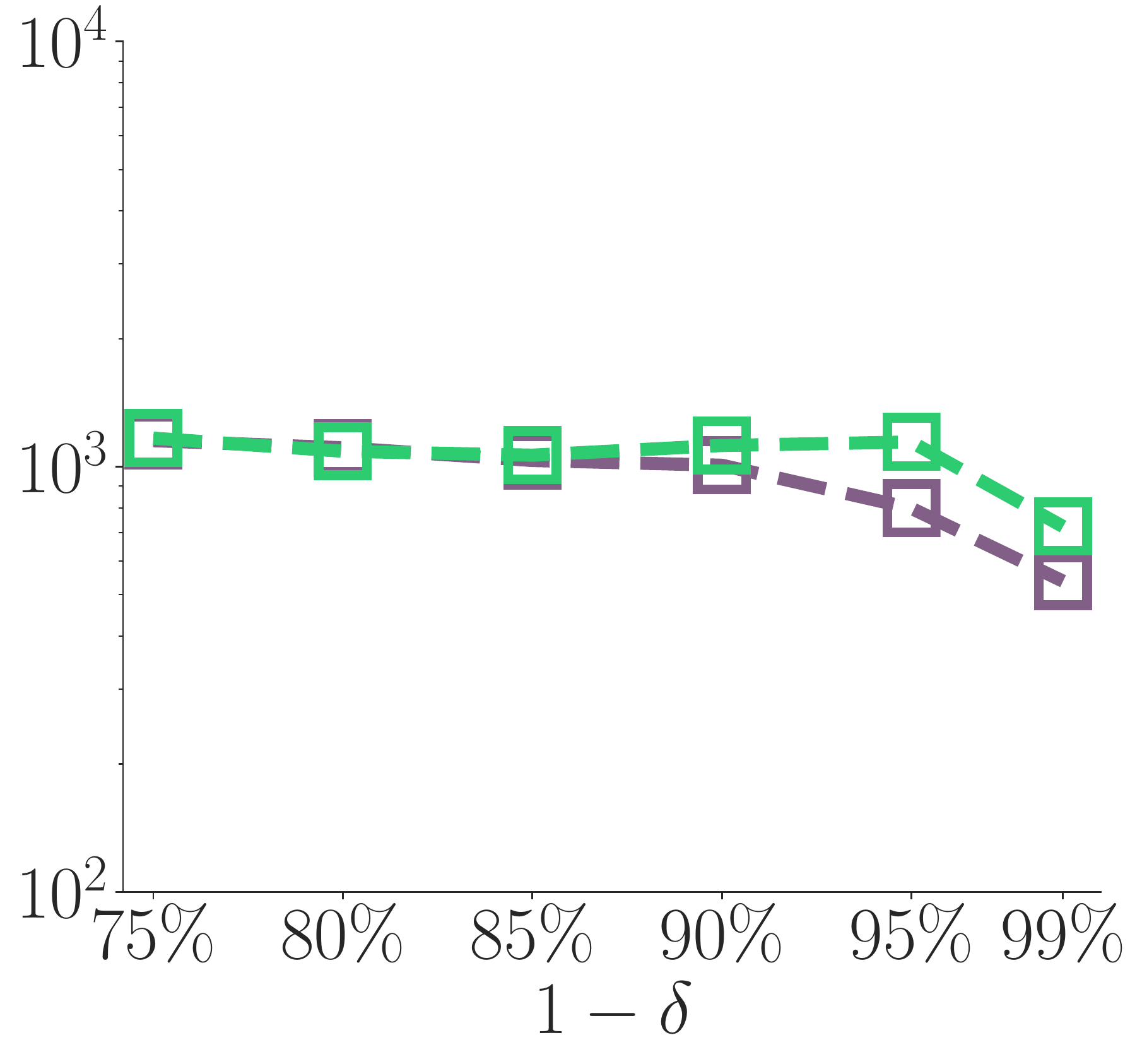}\\
         (d) \small{Diagnosis for $\hat Q_1$  with different $\delta$ } & (e) \small{Diagnosis for $\hat Q_2$ with different $\delta$ } & (f) Norm of  $Q_{\rm{debias}}$\\
    \end{tabular}
    \vspace{-.8em}
    \caption{Post-hoc diagnosis on Puck-Mountain. We set the discounted factor $\gamma=0.95$, the horizon length $T=100$, number of transitions $n=500$, failure probability $\delta=0.10$, temperature of the behavior policy $\tau = 1$, and the feature dimension 10 as default. The rest of the parameters are the same as that in Section \ref{sec:exp}.
    }
    \label{fig:puckmountain_debias}

\end{figure*}

\begin{figure}[t]
    \centering
     \begin{tabular}{ccc}
        \multicolumn{3}{c}{
        \includegraphics[width=.85\linewidth]{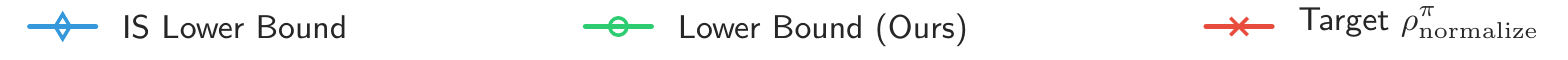}}\\
        \raisebox{2em}{\rotatebox{90}{\small{ Normalized Reward}}}\includegraphics[height=\delen]{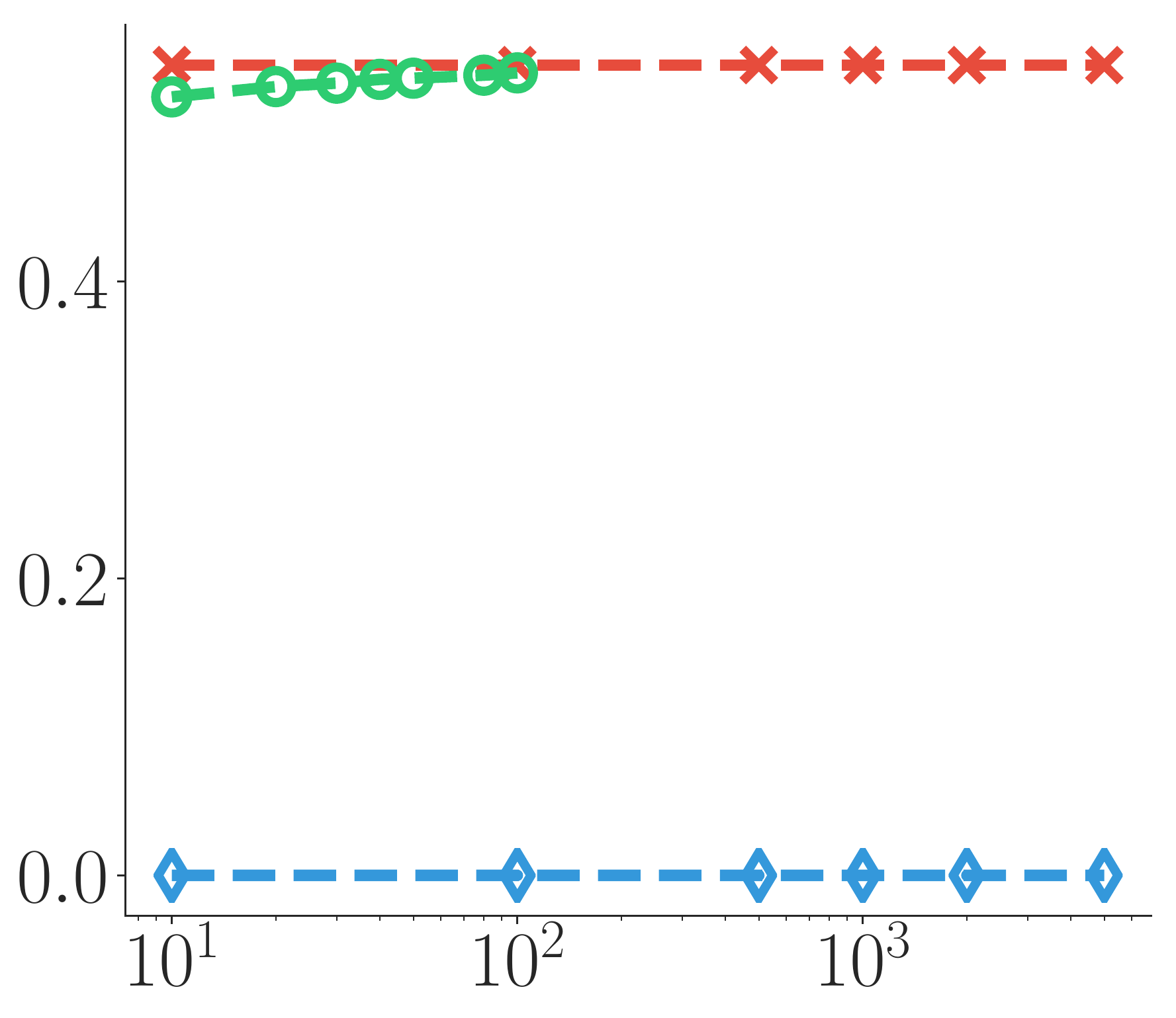}&
        \raisebox{3em}{\rotatebox{90}{\small{Normalized Rewards}}}\includegraphics[height=\delen]{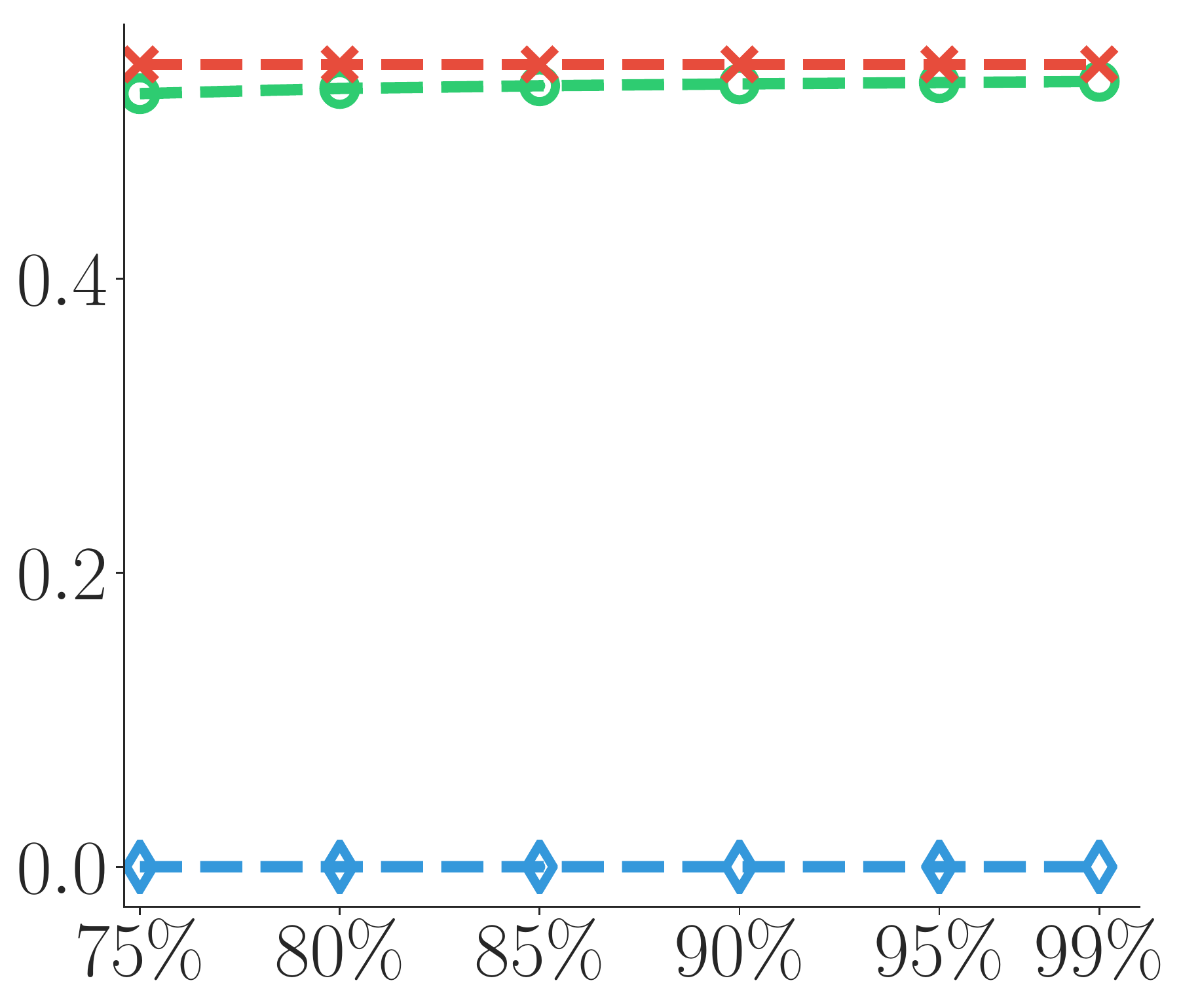} &
        \raisebox{3em}{\rotatebox{90}{\small{Normalized Rewards }}}\includegraphics[height=\delen]{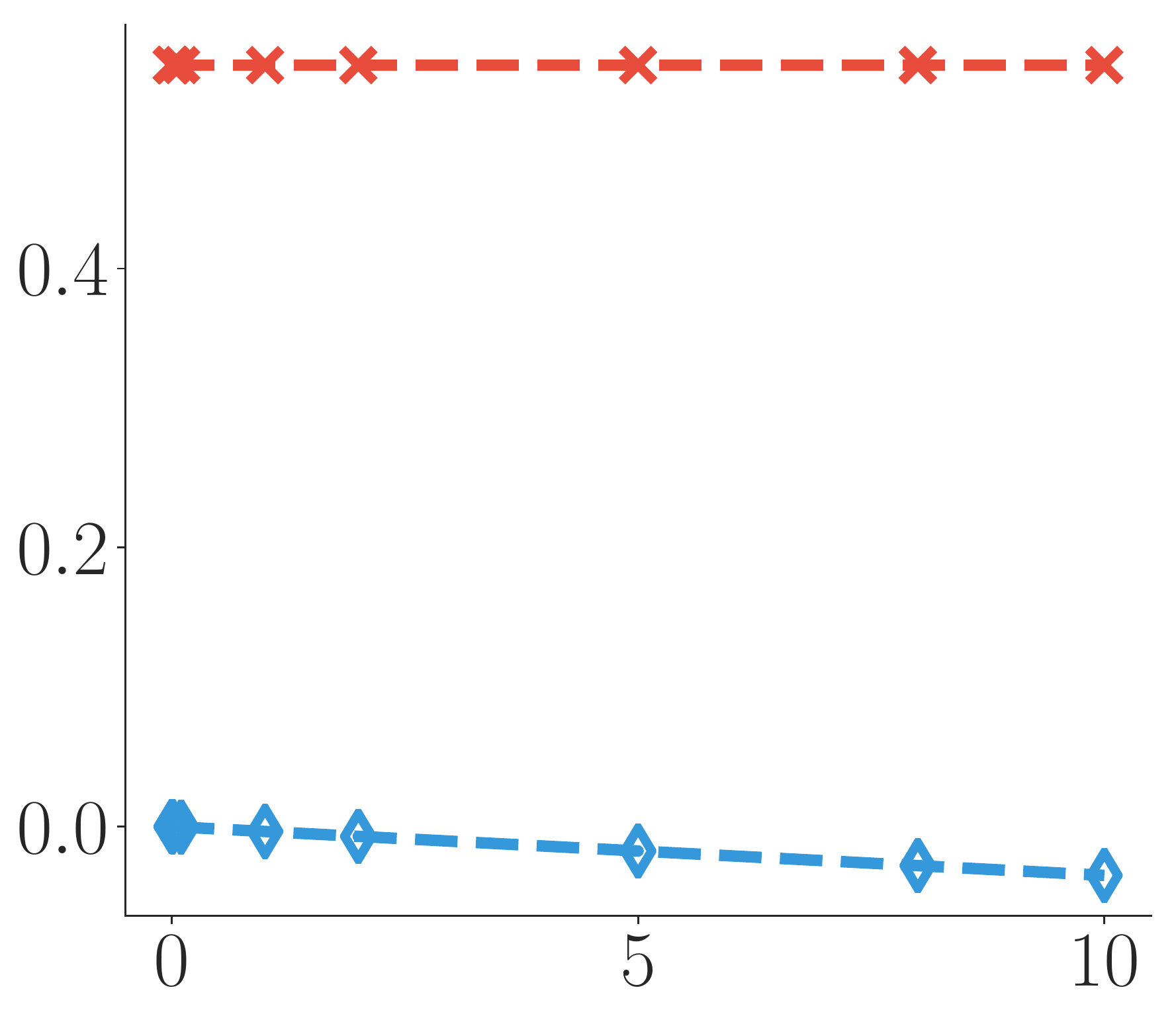}\\
        (a) Number of Episodes, $N$ & (b) $1 - \delta$ & (c) Threshold value $c$\\
    \end{tabular}
    
    \caption{IS-based confidence lower bounds \citet{thomas15high} on  Inverted-Pendulum following the same setting as that in Section \ref{sec:exp}. We use the results of random Fourier feature in Figure \ref{fig:pendulum} as our lower bound.
    For   \citet{thomas15high}, we set the threshold value to be $c=10^{-5}$ in (a) \& (b), and the number of episodes $N = 2000$ in (b) \& (c). 
    For our method, we only use  $n = 50 \times 20$ number of transition pairs in (b). 
    We report the normalized reward based on the procedure in \citet{thomas15high} (${\rho_{\rm{normalize}}} = (\sum_{t=1}^{T}\gamma^{t-1}r_{t} - R_{\min}) / (R_{\max} - R_{\min})$).  
    }
    \label{fig:thomas_bounds}

\end{figure}

\subsection{Additional Experiments}
\paragraph{Post-hoc Diagnosis Experiments on Puck-Mountain}
Figure \ref{fig:puckmountain_debias} (a)-(f) show the diagnosis results for two estimations of Q-function ($\hat Q_1$ and $\hat Q_2$) on Puck-Mountain. 
Here both $\hat Q_{1}$ and $\hat Q_{2}$ have relatively large bias, but $\hat Q_{1}$ tends to overestimate $\eta^{\pi}$ (see Figure~\ref{fig:puckmountain_debias}(a)), while $\hat Q_{2}$ tends to underestimate $\eta^{\pi}$ (see Figure~\ref{fig:puckmountain_debias}).

Figure \ref{fig:puckmountain_debias}(a)-(c) show that as we increase the number of transitions, the norm of the debiasing term $Q_{\rm{debias}}$ becomes larger.
This is because when we have more data, we have a relatively tight confidence bound and we need a more complex debias function to provide good post-hoc correction.
Figure \ref{fig:puckmountain_debias} (d)-(f) demonstrate the performance of our algorithms when we change the failure probability $\delta$. 

\paragraph{Comparison with \citet{thomas15high}}

As a comparison, we implement the method from \citet{thomas15high}, which uses concentration inequality to construct confidence bounds on an importance sampling (IS) based estimator.
Following \citet{thomas15high}, we assume the expected reward is normalized as follows 
\begin{align}
    \rho^{\pi}_{\rm{normalize}} := \frac{\E_{\pi}\left[\sum_{t=1}^{T}\gamma^{t-1}r_{t}\right] - R_{\min}}{R_{\max} - R_{\min}}\,, \label{equ:norm_return}
\end{align}
where $\sum_{t=1}^{T}\gamma^{t-1}r_{t}$ is the discounted return of a trajectory following policy $\pi$, and $R_{\max}$ and $R_{\min}$ are the upper and lower bounds on
$\sum_{t=1}^{T}\gamma^{t-1}r_t$; see \citet{thomas15high} for the choice of $R_{\max}$ and $R_{\min}$. 

Given a set of trajectories $\{(s_t^{(i)}, a_t^{(i)}, r_t^{(i)})_{t=1}^{T}\}_{1 \leq  i \leq N}$ generated by behavior policy $\pi_{0}$, we have 

\begin{align}
    \hat{\rho}^{\pi}_{\rm IS} := \frac{1}{N} \sum_{i=1}^{N} X_{i}, ~~~~~\text{with}~~~~~X_{i}= \underbrace{R_i}_{\text{return}}\underbrace{\prod_{t=1}^{T}\frac{\pi(a_t^{(i)}|s_t^{(i)})}{\pi_{0}(a_t^{(i)}|s_t^{(i)})}}_{\text{importance weight}}\,, \label{equ:is_est}
\end{align}

where $R_{i}$ is reward the $i$-th trajectory from the data 
(normalized as shown in Eq \eqref{equ:norm_return}). 
 Theorem 1 in \citet{thomas15high} provides a concentration inequality 
for constructing lower bound of  $\rho^{\pi}_{\rm{normalize}}$ based on a truncated importance sampling estimator. 
Let $\{c_i\}_{i=1}^N$ be a set of 
positive real-valued threshold and $\delta \in (0,1)$ and $Y_{i} = \min\{X_i, c_i\}$, %
we have with probability at least $1 - \delta$, %
\begin{align}
    \rho^{\pi}_{\rm{normalize}} \geq \underbrace{\left(\sum_{i=1}^{N}\frac{1}{c_i}\right)^{-1}\sum_{i=1}^{N}\frac{Y_i}{c_i}}_\text{empirical mean} - \underbrace{\left(\sum_{i=1}^{N}\frac{1}{c_i}\right)^{-1}\frac{7N\ln(2 / \delta)}{3(N-1)}}_\text{
term that goes to zero as $1 / N$ as $N \to \infty$} - \underbrace{\left(\sum_{i=1}^{N}\frac{1}{c_i}\right)^{-1}\sqrt{\frac{\ln(2/\delta)}{N-1}\sum_{i, j=1}^{N}\left(\frac{Y_i}{c_i} - \frac{Y_j}{c_j}\right)^2}}_\text{term that goes to zero as $1 / \sqrt{N}$ as $N \to \infty$}\,.\label{eq:thomas_bound}
\end{align}
The RHS provides a %
lower bound of $\rho^{\pi}_{\rm{normalized}}$ based the empirical trajectories. %
Following the settings in \citet{thomas15high}, we set the threshold values to be a constant $c$, that is, $c_i = c$ for all $i =1,\ldots, N$.

We  evaluate the method on Inverted-Pendulum under the same default settings as our experiments in the paper.  
Figure \ref{fig:thomas_bounds}(a)-(c) show the results of the high confidence lower bound.
We can see that the IS lower bounds are almost vacuous (i.e., very close to zero) and 
it does not help very much when we 
increase the number of episodes (up to $N=2000$) or the failure probability $\delta$. 
This is  because  there is only small overlap between the target policy $\pi$ and behavior policy $\pi_0$ in this case and when the horizon length  is large ($T = 50$), making  the IS  estimator degenerate. Although we need to point out that to get tighter lower bound with our method, we assume the true $Q^{\pi}$ is in the function space $\mathcal{F}$ that we choose, 
and also assume the horizon length $T\to \infty $ to make the confidence bound provably hold.

\clearpage
\section{Discussion on Reproducing Kernel Hilbert Space}
\label{sec:appendix_random_feature}
We provide proof of Proposition~\ref{pro:full_rkhs} and discussion for Section \ref{sec:main_rkhs}.

\begin{proof}[Proof of Proposition~\ref{pro:full_rkhs}]
 Consider the Lagrangian of the constrained optimization \eqref{eq:optimization-problem}, we have:
\begin{align*}
    L(Q, \lambda_1, \lambda_2) =& \E_{x\sim \mu_{0}\times \pi}[Q(x)] - \lambda_1 (\hat{L}_K(Q) - \lambda_K) - \lambda_2 (\|Q\|_{\H}^2 - \rho) \\
    =& \langle Q, f_0 \rangle - \frac{\lambda_1}{n^2}\left(\sum_{i,j} (\langle Q, f_i\rangle - r_i ) K(x_i,x_j) (\langle Q, f_j\rangle - r_j ) \right) - \lambda_2(\|Q\|_{\H}^2 ) - C,
\end{align*}
where $\lambda_1,\lambda_2$ are Lagrangian multipliers with respect to the two constraints, and $C$ is a constant related to $\lambda_1, \lambda_2, \lambda_K$ and $\rho$.
We rewrite $Q$ into 
$$
Q = \sum_{k=0}^n \alpha_k f_k + Q_{\perp},
$$
where $Q_\perp$ is in the orthogonal subspace to the linear span of $f_0, f_1,...,f_n$, that is, $\langle f_i, Q_\perp\rangle = 0,~\forall i\in [n]$.
By 
decomposin $Q$ into $\sum_{k=0}^n \alpha_k f_k$ and $Q_{\perp}$, we have
\begin{align*}
    L(Q, \lambda_1, \lambda_2) + C =& \langle Q, f_0 \rangle - \frac{\lambda_1}{n^2}\left(\sum_{i,j} (\langle Q, f_i\rangle - r_i ) K(x_i,x_j) (\langle Q, f_j\rangle - r_j )\right) - \lambda_2(\|Q\|_{\H}^2) \\
    =& \langle \sum_{k=0}^n \alpha_k f_k + Q_{\perp}, f_0 \rangle - \frac{\lambda_1}{n^2}\left(\sum_{i,j=1}^n (\langle \sum_{k=0}^n \alpha_k f_k + Q_{\perp}, f_i\rangle - r_i ) K(x_i,x_j) (\langle \sum_{k=0}^n \alpha_k f_k + Q_{\perp}, f_j\rangle - r_j )\right)\\
    ~~~~~~~~& - \lambda_2(\|\sum_{k=0}^n \alpha_k f_k\|_{\H}^2 +\|Q_{\perp}\|_{\H}^2) \\
    =& \langle \sum_{k=0}^n \alpha_k f_k, f_0 \rangle - \frac{\lambda_1}{n^2} \left(\sum_{i,j=1}^n (\langle \sum_{k=0}^n \alpha_k f_k , f_i\rangle - r_i ) K(x_i,x_j) (\langle \sum_{k=0}^n \alpha_k f_k , f_j\rangle - r_j )\right)\\
    ~~~~~~&- \lambda_2(\|\sum_{k=0}^n \alpha_k f_k\|_{\H}^2 + \|Q_{\perp}\|_{\H}^2 ) \\
    \leq& \langle \sum_{k=0}^n \alpha_k f_k, f_0 \rangle - \frac{\lambda_1}{n^2} \left(\sum_{i,j=1}^n (\langle \sum_{k=0}^n \alpha_k f_k , f_i\rangle - r_i ) K(x_i,x_j) (\langle \sum_{k=0}^n \alpha_k f_k , f_j\rangle - r_j )\ell\right)\\
    ~~~~~~&- \lambda_2(\|\sum_{k=0}^n \alpha_k f_k\|_{\H}^2) :=L(\alpha, \lambda_1, \lambda_2),
\end{align*}
where the optimum $Q$ will have $Q_{\perp} = 0$ and $L(\alpha, \lambda_1, \lambda_2)$ is the Lagrangian w.r.t. to coefficient $\alpha$.
Collecting the term we can reform the optimization w.r.t. $\alpha$ as
\begin{align*}
    \hat{\eta}^+ :=  \max_{\{\alpha_i\}_{0\leq i \leq n}} \Big \{  [c^\top \alpha  + \lambda_{\eta}],
    ~~~~~~~~~\text{s.t.} ~~ \alpha^\top A \alpha + b^\top \alpha + d \leq 0,%
    ~~~~~~~~~  \alpha^\top B \alpha \leq \rho. \Big\}\, 
\end{align*}
where $B_{ij} = \langle f_i, f_j\rangle_{\H_{\Kf}},~\forall 0 \leq i, j \leq n$ is the inner product matrix of $\{f_i\}_{0\leq i\leq n}$ under $\H_{\Kf}$, $c$ as the first column of $B$ and $B_1$ be the remaining sub-matrix.
Let $M_{ij} = K(x_i,x_j)$ be the kernel matrix, $R_{i} = r_i$ be the vector of reward from data, then
$A = \frac{1}{n^2} B_1 M B_1^\top$, $b = -\frac{1}{n^2} B_1 M R$ and $d = \frac{1}{n^2} R^\top M R - \lambda_K$.
\end{proof}

\paragraph{Random Feature Approximation} 
Consider the random feature representation of  kernel $\Kf$ 
\begin{align}  \label{equ:K0app}
\Kf(x,x') = \E_{w\sim \mu}[\phi(x,w) \phi(x',w)],
\end{align}
where $\mu$ is a distribution of random variable $w$. 
Every function  $f$ in the RKHS $\H_{K_0}$ associated with $K_0$ can be represented as 
$$
f(x) = \E_{w\sim \mu}[\phi(x,w) \alpha_f(w)],
$$
whose RKHS norm equals 
$$
\|f\|_{\H_{K_0}}^2 = \E_{w\sim \mu}[\alpha_f(w)^2].
$$

To approximate $\Kf$, we draw an i.i.d. sample $\{w_i\}$ from $\mu$ to approximate $\Kf$ with  %
\begin{align} \label{equ:hatK0app}
\widehat{\Kf}(x,x') = \frac{1}{m} \sum_{i=1}^m \phi(x,w_i) \phi(x',w_i).
\end{align}
Similarly, each $f \in \H_{\widehat{\Kf}}$ can be represented as 
$$
f(x) = \frac{1}{m}\sum_{i=1}^m \phi(x,w_i) \alpha_i
$$
and its corresponding RKHS norm is 
$$
\|f\|_{\H_{\widehat{\Kf}}}^2 = \frac{1}{m} \sum_{i=1}^m \alpha_i^2.
$$
Denote by $\theta = \alpha/m$, then we have $f(x) = \theta^\top \Phi(x)$ and $\norm{f}^2_{\H_{\widehat{\Kf}}} = m \norm{\theta}_2^2$, which is the form used in the paper. 

The result below shows that when $Q^\pi$ is included in $\H_{K_0}$ but may not be included in $\H_{\hat K_0}$, we can still get a provably upper bound by setting the radius of the optimization domain properly large.  

\begin{thm}\label{thm:finte_random_feature_error}
Let $K_0$ be a positive definite kernel with random feature expansion in \eqref{equ:K0app} and $\hat K_0$ defined in \eqref{equ:hatK0app} with $\{w_i\}_{i=1}^m$ i.i.d. drawn from $\mu$. Assume $\norm{\phi}_{\infty} = \sup_{x,w}|\phi(x,w)| <0$. 
Define 
\begin{align} 
 & \eta^+_{\Kf}  = 
 \max_{Q\in \H_{{\Kf}}}\{\eta(Q),~~~~\text{s.t.}~~~~ \hat{L}_K(Q) \leq \lambda, ~~~~\|Q\|_{\H_{{\Kf}}}^2\leq \rho\}. \label{equ:pp1} 
 \end{align}
 Let $Q^*$ be the optimal solution of \eqref{equ:pp1} and $Q^*(\cdot) = \E_{w\sim \mu}[\phi(\cdot,w)\alpha^*(w)]$. Assume $C := \var_{w\sim \mu}((\alpha^*(w))^2)<\infty$. %
 For $\delta \in (0,1)$, define 
 \begin{align} 
 & \eta^+_{\widehat{\Kf}}  = 
 \max_{Q\in \H_{\widehat{\Kf}}}\{\eta(Q),~~~~\text{s.t.}~~~~ \hat{L}_K(Q) \leq \lambda, ~~~~\|Q\|_{\H_{\widehat{\Kf}}}^2\leq \rho' \} .  \label{equ:pp2}
\end{align}
If we set $\rho' \geq \rho + \sqrt{\frac{C}{\delta m}}$, then we have with probability at least $1-2\delta$ 
\begin{equation*}
 \eta^+_{\Kf}  \leq  \eta^+_{\widehat{\Kf}}  + \|\phi\|_{\infty} \sqrt{\frac{ \rho }{\delta m}}. 
\end{equation*}
Therefore, 
if $Q^\pi$ belongs to $\H_{K_0}$ (and hence $\eta^\pi \leq \eta^+_{K_0}$), 
then $\eta^+_{\widehat{\Kf}}  + \|\phi\|_{\infty} \sqrt{\frac{ \rho }{\delta m}}$ provides a high probability upper bound of $\eta^\pi$. 
\end{thm}
\begin{proof}
Following 
$Q^*(\cdot) = \E_{w\sim \mu}[\phi(\cdot,w)\alpha^*(w)]$, we have $\alpha^*$ satisfies  
\begin{align}
    \label{equ:rtreitrt}
    \norm{Q^*}_{\H_{K_0}}^2 = \E_{w\sim \mu} \left [ ( \alpha^*(w) )^2 \right ] \leq \rho. 
\end{align}
Let %
$$
\widetilde{Q}(\cdot) = \frac{1}{m} \sum_{i=1}^m \phi(\cdot, w_i) \alpha^*(w_i), 
$$
for which we have 
$$
\|{\widetilde{Q}}\|_{\H_{\hat K_0}}^2 = \frac{1}{m} \sum_{i=1}^m  \alpha^*(w_i)^2. 
$$
By Chebyshev inequality, we have with probability at least $1-\delta$ 
$$
\|{\widetilde{Q}}\|_{\H_{\hat K_0}}^2 \leq 
\norm{Q^*}_{\H_{K_0}}^2  + \sqrt{\frac{1}{\delta m}\var_{w\sim \mu}(\alpha^*(w)^2)} \leq \rho + \sqrt{\frac{C}{\delta m}}. 
$$
If $\rho' \geq \rho + \sqrt{\frac{C}{\delta m}}$, 
then $\widetilde{Q}$ is included in the optimization set of \eqref{equ:pp2}, and hence \begin{align} \label{equ:tmpdfdfdfjidjf}
\eta(\widetilde{Q}) \leq \eta^+_{\hat K_0}.
\end{align}

On the other hand, 
because $\eta(Q)$ is a linear functional, we have 
\begin{align*}
    \eta(\widetilde{Q}) - \eta(Q^*)
    =& \frac{1}{m}\sum_{i=1}^m \Phi(w_i) \alpha^*(w_i) - \E_{w\sim \mu}[\Phi(w_i)\alpha^*(w)],
\end{align*}
where $\Phi(w) = \eta(\phi(\cdot,w))$. Note that 
$$
\var_{w\sim \mu}(\Phi(w) \alpha^*(w)) \leq \norm{\phi}_\infty^2  \E_{w\sim \mu} \left [(\alpha^*(w))^2\right] \leq \norm{\phi}_\infty^2 \rho, 
$$%
where we used \eqref{equ:rtreitrt}.  
By Chebyshev inequality, we have with probability at least $1-\delta$ 
$$
|\eta(\widetilde{Q}) - \eta(Q^*)| \leq \sqrt{\frac{\|\phi\|_{\infty}^2 \rho }{\delta m}}.
$$
Combining this with \eqref{equ:tmpdfdfdfjidjf}, we have, with probability at least $1-2\delta$, 
\begin{align*} 
 \eta^+_{K_0} = \eta(Q^*)
& =  \eta(\widetilde{Q})~+~ (\eta(Q^*) - \eta(\widetilde{Q}))  
  \leq  \eta^+_{\hat K_0}  ~ + ~\sqrt{\frac{\|\phi\|_{\infty}^2 \rho }{\delta m}}. 
\end{align*}
\end{proof}

\end{document}